%% file: main.tex
\pdfoutput=1

\documentclass[twoside,11pt]{article} 
\usepackage{jmlr2e}

\usepackage{amsmath}

\title{Differentially Private Testing of Identity and Closeness \\ of Discrete Distributions}

\author{Jayadev Acharya\footnote{The authors are listed in alphabetical order. 
This research is supported by the National Science Foundation NSF-CRII-1657471, and a start-up grant from Cornell University.}\\
Cornell University\\
\tt{acharya@cornell.edu}
\and
Ziteng Sun$^*$\\
Cornell University\\
\tt{zs335@cornell.edu}
\and
Huanyu Zhang$^*$\\
Cornell University\\
\tt{hz388@cornell.edu}
}

\input{glodef}

\input{locdef}

\firstpageno{1}
\begin{document}

\title{Differentially Private Testing of Identity and Closeness of Discrete Distributions}

\author{\name Jayadev Acharya \email acharya@cornell.edu\\ 
\name Ziteng Sun \email zs335@cornell.edu\\ 
\name Huanyu Zhang \email hz388@cornell.edu\\
\addr Cornell University\\ 
Ithaca, NY}

\editor{}

\maketitle

\input{abstract.tex}


%
%
%

\newcommand{\pb}[2]{\parbox[c][][c]{#1}{\strut#2\strut}}

\renewcommand*{\arraystretch}{1.3}

\input{intro}

\input{contributions.tex}
\input{table-new}

\input{related-work.tex}

\input{organization.tex}

\input{preliminaries.tex}

\input{coupling}
\input{binary-upper}
\input{binary-lower-coupling}



\input{uniformity}

\input{uniformity-to-identity}

\input{uniformity-upper_new2.tex}

\input{uniformity-lower_new.tex}

\input{closeness-testing.tex}

\bibliographystyle{alpha}

\bibliography{abr,masterref}

\appendix

\input{app-binary-upper.tex}

\input{app-uniformity-to-identity.tex}

\input{uniformity-upper.tex}

\input{uniformity-upper-lemma.tex}

\input{lower-bound-coupling.tex}

\input{lower-bound-coupling-medium.tex}

\end{document}

%% file: glodef.tex
\usepackage{lmodern}
\usepackage{xspace}
\usepackage[protrusion=true,expansion=true]{microtype}
\usepackage{fullpage}

\usepackage{multirow}

\usepackage{dsfont} 
\usepackage{fullpage}

\usepackage{algorithmicx,algpseudocode,  algorithm}

\usepackage[shortlabels]{enumitem}
\setitemize{noitemsep,topsep=0pt,parsep=0pt,partopsep=0pt}
\setenumerate{noitemsep,topsep=0pt,parsep=0pt,partopsep=0pt}

\makeatletter
\@ifundefined{theorem}{
  \theoremstyle{definition}
  \newtheorem{definition}{Definition}
  \theoremstyle{plain}
  \newtheorem{theorem}{Theorem}
  
  \newtheorem{lemma}{Lemma}

  \theoremstyle{remark}

  \newtheorem{example}{Example}
  
}{}
\makeatother

\newcommand{\ignore}[1]{}

\newcommand{\II}{\mathbb{I}} 
\newcommand{\EE}{\mathbb{E}}

\newcommand{\NN}{\mathbb{N}}

\newcommand{\RR}{\mathbb{R}}

\newcommand{\expectation}[1]{\EE\left[#1\right]}

\def \cA     {{\cal A}}

\def \cX     {{\cal X}}
\def \cY     {{\cal Y}}

\newcommand{\Var}{{\rm Var}}

\def \upto  {{,}\ldots{,}}

\def \ceil#1{{\lceil{#1}\rceil}}

\newcommand{\absv}[1]{\left|#1\right|}

\def \Paren#1{{\left({#1}\right)}}

\def \Brack#1{{\left[{#1}\right]}}

\newcommand{\ed}{\stackrel{\Delta}{=}}

\def \half    {{\frac12}}

\newcommand{\probof}[1]{\Pr\Paren{#1}}

\def\ignore#1{}

\newcommand{\bi}{\begin{itemize}}
\newcommand{\ei}{\end{itemize}}

\def\orpro{\mathop{\mathchoice
   {\vee\kern-.49em\raise.7ex\hbox{$\cdot$}\kern.4em}
   {\vee\kern-.45em\raise.63ex\hbox{$\cdot$}\kern.2em}
   {\vee\kern-.4em\raise.3ex\hbox{$\cdot$}\kern.1em}
   {\vee\kern-.35em\raise2.2ex\hbox{$\cdot$}\kern.1em}}\limits}

\def\andpro{\mathop{\mathchoice
 {\wedge\kern-.46em\lower.69ex\hbox{$\cdot$}\kern.3em}
 {\wedge\kern-.46em\lower.58ex\hbox{$\cdot$}\kern.25em}
 {\wedge\kern-.38em\lower.5ex\hbox{$\cdot$}\kern.1em}
 {\wedge\kern-.3em\lower.5ex\hbox{$\cdot$}\kern.1em}}\limits}

\def\simge{\mathrel{%
   \rlap{\raise 0.511ex \hbox{$>$}}{\lower 0.511ex \hbox{$\sim$}}}}

\def\simle{\mathrel{
   \rlap{\raise 0.511ex \hbox{$<$}}{\lower 0.511ex \hbox{$\sim$}}}}



%% file: locdef.tex
\newcommand{\absz}{k}
\newcommand{\ab}{k}

\newcommand{\ns}{m}

\newcommand{\p}{p}
\newcommand{\q}{q}

\newcommand{\Sitnp}{S({\tt IT}, \ab, \dist)}

\newcommand{\Sit}{S({\tt IT}, \ab, \dist,\eps)}
\newcommand{\Sitd}{S({\tt IT}, \ab, \dist,\eps)}
\newcommand{\Sut}[2]{S({\tt UT}, #1, #2,\eps)}
\newcommand{\Sutnp}[2]{S({\tt UT}, #1, #2)}

\newcommand{\Sct}{S({\tt CT}, \ab, \dist,\eps)}
\newcommand{\Sctd}{S({\tt CT}, \ab, \dist,\eps,\delta)}
\newcommand{\Sctnp}{S({\tt CT}, \ab, \dist)}

\newcommand{\unif}{u}
\newcommand{\unifab}{\unif[{\ab}]}
\newcommand{\unifabs}[1]{\unif[{#1}]}

\newcommand{\smb}{x}
\newcommand{\psmb}{\p_\smb}

\newcommand{\Xon}{X_1^\ns}
\newcommand{\Yon}{Y_1^\ns}
\newcommand{\Xot}{X_1^t}
\newcommand{\Yot}{Y_1^t}
\newcommand{\Xos}{X_1^s}
\newcommand{\Xosp}{X_1^{'s}}
\newcommand{\Yos}{Y_1^s}
\newcommand{\Yosp}{Y_1^{'s}}
\newcommand{\Xosmo}{X_1^{s-1}}
\newcommand{\Yosmo}{Y_1^{s-1}}

\newcommand{\eps}{\varepsilon}

\newcommand{\mltx}[1]{\mu_{#1}}
\newcommand{\mlty}[1]{\nu_{#1}}

\newcommand{\bias}{b}
\newcommand{\Bern}[1]{B(#1)}

\newcommand{\ham}[2]{d(#1,#2)}
\newcommand{\alg}{\cA}

\newcommand{\Mltsmb}[2]{M_{#1}(#2)}
\newcommand{\prfsmb}[2]{\Phi_{#1}(#2)}

\newcommand{\dist}{\alpha}

\newcommand{\dtv}[2]{d_{TV}(#1, #2)}
\newcommand{\bino}[2]{{\rm Bin}(#1, #2)}

%% file: abstract.tex
\begin{abstract}
We study the fundamental problems of identity testing (goodness of fit), and closeness testing (two sample test) of distributions over $k$ elements, under differential privacy. While the problems have a long history in statistics,  finite sample bounds for these problems have only been established recently. 

In this work, we derive upper and lower bounds on the sample complexity of both the problems under $(\varepsilon, \delta)$-differential privacy. We provide optimal sample complexity algorithms for identity testing problem for all parameter ranges, and  the first results for closeness testing. Our closeness testing bounds are optimal in the sparse regime where the number of samples is at most $k$. 

Our upper bounds are obtained by privatizing non-private estimators for these problems. The non-private estimators are chosen to have small sensitivity. We propose a general framework to establish lower bounds on the sample complexity of statistical tasks under differential privacy. We show a bound on differentially private algorithms in terms of a coupling between the two hypothesis classes we aim to test. By constructing carefully chosen priors over the hypothesis classes, and using Le Cam's two point theorem we provide a general mechanism for proving lower bounds.  We believe that the framework can be used to obtain strong lower bounds for other statistical tasks under privacy.
\end{abstract}

%% file: intro.tex
\section{Introduction}

Testing whether observed data conforms to an underlying model is a fundamental scientific problem. In a statistical framework, given samples from an unknown probabilistic model, the goal is to determine whether the underlying  model has a property of interest. 

This question has received great attention in statistics as hypothesis testing~\cite{NeymanP33, lehmann2006testing}, where it was mostly studied in the asymptotic regime when the number of samples $\ns\to\infty$. In the past two decades there has been a lot of work from the computer science, information theory, and statistics community on various distribution testing problems in the non-asymptotic (small-sample) regime, where the domain size $k$ could be potentially larger than $\ns$ (See~\cite{BatuFRSW00, BatuFFKRW01, GoldreichR00, Batu01, Paninski08, AcharyaJOS13b, AcharyaJOS14b, chan2014optimal, DiakonikolasKN14, bhattacharya2015testing, canonne2016testing, diakonikolas2016new, BatuC17}, references therein, and~\cite{Canonne15} for a recent survey). Here the goal is to characterize the minimum number of samples necessary (sample complexity) as a function of the domain size $\ab$, and the other parameters.

At the same time, preserving the privacy of individuals who contribute to the data samples has emerged as one of the key challenges in designing statistical mechanisms over the last few years. For example, the privacy of individuals participating in surveys on sensitive subjects is of utmost importance. 
Without a properly designed mechanism, statistical processing might divulge the sensitive information about the data. 
There have been many publicized instances of individual data being de-anonymized, including the deanonymization of Netflix database~\cite{narayanan2008robust}, and individual information from census-related data~\cite{sweeney2002k}. Protecting privacy for the purposes of data release, or even computation on data has been studied extensively across several fields, including statistics, machine learning, database theory, algorithm design, and cryptography (See e.g.,~\cite{warner1965randomized, Dalenius77, DinurN03, wasserman2010statistical, duchi2013local, wainwright2012privacy, chaudhuri2011differentially}). While the motivation is clear, even a formal notion of privacy is not straight-forward. We use~\emph{differential privacy}~\cite{DworkMNS06}, a notion which rose from database and cryptography literature, and has emerged as one of the most popular privacy measures (See~\cite{DworkMNS06, Dwork08, wasserman2010statistical, dwork2010boosting, blum2013learning, mcsherry2007mechanism, li2015matrix, kairouz2017composition}, references therein, and  the recent book~\cite{DworkR14}). 
Roughly speaking, it requires that the output of the algorithm should be statistically close on two neighboring datasets. For a formal definition of differential privacy, see Section~\ref{sec:preliminaries}. 

A natural question when designing a differentially private algorithm is to understand how the data requirement grows to ensure privacy, along with the same accuracy level. In this paper, we study the sample size requirements for differentially private discrete distribution testing. 


%% file: contributions.tex
\subsection{Results and Techniques}
We consider two fundamental statistical tasks for testing distributions over $[\ab]$: (i) identity testing, where given sample access to an unknown distribution $\p$, and a known distribution $\q$, the goal is to decide whether $\p=\q$, or $\dtv{\p}{\q}\ge\dist$, and (ii) closeness testing, where given sample access to   unknown distributions $\p$, and $\q$, the goal is to decide whether $\p=\q$, or $\dtv{\p}{\q}\ge\dist$. (See Section~\ref{sec:preliminaries} for precise statements of these problems).
Given differential privacy constraints $(\eps,\delta)$, we provide $(\eps,\delta)$-differentially private algorithms for both these tasks. For identity testing, our bounds are optimal up to constant factors for all ranges of $\ab, \alpha, \eps, \delta$, and for closeness testing the results are tight in the small sample regime where $\ns=O(\ab)$. Our upper bounds are based on various methods to privatize the previously known tests. A critical component is to design and analyze test statistic that have low sensitivity, in order to preserve privacy. 

We first state that any $(\eps+\delta,0)$-DP algorithm is also an $(\eps, \delta)$ algorithm.~\cite{CaiDK17} showed that for testing problems, any $(\eps, \delta)$ algorithm will also imply a $(\eps+ c\delta, 0)$-DP algorithm. Therefore, for all the problems, we simply consider $(\eps,0)$-DP algorithms, and we can replace $\eps$ with $\eps+\delta$ in both the upper and lower bounds without loss of generality. 

One of the main contributions of our work is to propose a general framework for establishing lower bounds the sample complexity of statistical problems such as property estimation and hypothesis testing under privacy constraints. We describe this, and the other results below. A summary of the results is presented in Table~\ref{fig:badass:table}, which we now describe in detail.
\begin{enumerate}
\item
{\bf DP Lower Bounds via Coupling.} We establish a general method to prove lower bounds for distribution testing problems. Suppose $\Xon$, and $\Yon$ are generated by two statistical sources. Further suppose there is a coupling between the two sources such that the  expected hamming distance between the coupled samples is at most $D$, then if $D = o(1/(\eps+\delta))$, there is no $(\eps,\delta)$-differentially private algorithm to distinguish between the two sources. This result is stated precisely in Theorem~\ref{thm:coupling}.  

Using carefully designed coupling schemes, we provide lower bounds for binary testing, identity testing, and closeness testing. 

\item 
{\bf Binary Testing.} To study the problem of identity testing, we warm up with the binary testing problem, where $\ab=2$. The sample complexity of this problem is $\Theta(\frac1{\dist^2}+\frac1{\dist\eps})$. The upper bound is extremely simple, and can be obtained by the Laplace mechanism~\cite{DworkR14}, and the lower bound follows as an application of our general lower bound argument. The result is stated in Theorem~\ref{thm:bina-main}.
We construct a coupling between binary distributions, and apply Theorem~\ref{thm:coupling} to obtain a lower bound of $\Omega(\frac1{\dist^2}+\frac1{\dist\eps})$ samples for binary testing problem, which is tight up to a constant factor.
\item
{\bf Reduction from identity to uniformity.} We reduce the problem of $\eps$-DP identity testing of distributions over $[\ab]$ to $\eps$-DP uniformity testing over distributions over $[6\ab]$. Such a reduction, without privacy constraints was shown in~\cite{goldreich2016uniform}, and we use their result to obtain a reduction that also preserves privacy, with at most a constant factor blow-up in the sample complexity. This result is given in Theorem~\ref{thm:unif-identity}.  
\item
{\bf Identity Testing.} It was recently shown that $O(\frac{\sqrt\ab}{\alpha^2})$~\cite{Paninski08, ValiantV14, DiakonikolasKN14, adk15} samples are necessary and sufficient for identity testing without privacy constraints. The statistic used in these papers are variants of chi-squared tests, which could have a high global sensitivity. 

Given the reduction from identity to uniformity, it suffices to consider the statistic in~\cite{Paninski08} for uniformity testing. We show that privatizing this statistic yields a sample optimal testing algorithm with sample complexity $O\Paren{ \frac{\sqrt \ab }{\dist^2}+ \frac{\sqrt \ab} {\dist\sqrt\eps}}$, in the sparse regime where $\ns\le\ab$. This result is stated in Section~\ref{sec:upper-indentity}. However, Paninski's test fails when $\ns=\Omega(\ab)$. We therefore consider the test statistic studied by~\cite{diakonikolas2017sample} which is simply the distance of the empirical distribution to the uniform distribution. This statistic also has a low sensitivity, and futhermore has the optimal sample complexity in all parameter ranges, without privacy constraints. In Theorem~\ref{thm:main-identity}, we state the optimal sample complexity of identity testing. The upper bounds are derived by privatizing the statistic in~\cite{diakonikolas2017sample}.  We design a coupling between the uniform  distribution $\unifab$, and a mixture of distributions, which are all at distance $\alpha$ from $\unifab$ in total variation distance. In particular, we consider the mixture distribution used in~\cite{Paninski08}. Much of the technical details go into proving the existence of couplings with small expected Hamming distance.~\cite{CaiDK17} studied identity testing under pure differential privacy, and obtained an algorithm with complexity $O\Paren{ \frac{\sqrt \ab}{\dist^2}+ \frac{\sqrt {\ab\log\ab} } {\dist^{3/2} \eps}+ \frac {(k\log\ab)^{1/3} } {\dist^{5/3}\eps^{2/3}}}$. Our bounds improve their bounds significantly.

\item
{\bf Closeness Testing.}
Closeness testing problem was proposed by~\cite{BatuFRSW00}, and optimal bound of $\Theta\Paren{\max\{\frac{\ab^{2/3}}{\dist^{4/3}}, \frac{\ab^{1/2}}{\dist^2}\}}$ was shown in~\cite{chan2014optimal}. They proposed a chi-square based statistic, which we show has a small sensitivity. We privatize their algorithm to obtain the sample complexity bounds. In the sparse regime we prove a sample complexity bound of $\Theta\Paren{\frac{\ab^{2/3}}{\alpha^{4/3}}+\frac{\sqrt \ab}{\alpha\sqrt \eps}}$, and in the dense regime, we obtain a bound of $O\Paren{\frac{\sqrt \ab}{\alpha^2}+\frac1{\dist^2\eps}}$. These results are stated in Theorem~\ref{thm:close-main}. 
Since closeness testing is a harder problem than identity testing, all the lower bounds from identity testing port over to closeness testing. The closeness testing lower bounds are given in Theorem~\ref{thm:close-main}.


\end{enumerate}

%% file: table-new.tex
\begin{table}[t]
\centering
      \begin{tabular}{| l | c | }
      \hline
      \pb{0.20\textwidth}{{\bf Problem}} & {\bf Sample Complexity Bounds}  \\ \hline
      {\pb{0.20\textwidth}{\centerline{Is $\p = \Bern{\half}$?}}} & {\bf Non-private: $\Theta\Paren{\frac1{\alpha^2}} $}\\
      & {\bf $\eps$-DP: }$\Theta\Paren{\frac1{\alpha^2} +\frac1{\alpha\eps} }$ [\cite{DworkR14}, and Theorem~\ref{thm:bina-main}]\\ \hline
      {\pb{0.20\textwidth}{{{\bf Identity Testing}}}} & {\bf Non-private : $\Theta\Paren{\frac{\sqrt \ab}{\alpha^2}} $ \cite{Paninski08}}\\
              & {\bf $\eps$-DP algorithms:} $O\Paren{ \frac{\sqrt \ab}{\alpha^2}+\frac{\sqrt {\ab\log k}}{\alpha^{3/2}\eps}}$~\cite{CaiDK17}\\
       & {\bf $(\eps,\delta)$-DP algorithms for any $\eps,\delta$}\\
       &$\Sitd = \Theta\Paren{\frac{\sqrt \ab}{\alpha^2}+{\max\left\{\frac{\ab^{1/2}}{\alpha(\eps+\delta)^{1/2}}, \frac{\ab^{1/3}}{\alpha^{4/3}(\eps+\delta)^{2/3}}, \frac1{\dist(\eps+\delta)}\right\}}}$[Theorem~\ref{thm:main-identity}]\\\hline
            {\pb{0.20\textwidth}{{{\bf Closeness Testing}}}} & {\bf Non-private:} {$\Theta{ \left(\frac{\ab^{2/3}}{\dist^{4/3}}+ \frac{\ab^{1/2}}{\dist^2}\right) }$ \cite{chan2014optimal}}\\
       & {\bf $\eps$-DP:}\\
              & \multicolumn{1}{|l|}{{\bf IF} $\dist^2=\Omega\Paren{\frac1{\sqrt\ab}} \textbf{ and } \dist^2\eps=\Omega\Paren{\frac1k}$}   \\  
             & $\Sct = \Theta\Paren{\frac{\ab^{2/3}}{\alpha^{4/3}}+\frac{\sqrt \ab}{\alpha\sqrt \eps}}$ [Theorem~\ref{thm:close-main}]\\ 
                    & \multicolumn{1}{|l|}{{\bf ELSE}} \\
       & $\Omega\Paren{
       	\frac{\sqrt \ab}{\alpha^2}+\frac{\sqrt \ab}{\alpha\sqrt\eps}+\frac{1}{\alpha \eps}}\le \Sct \le O\Paren{\frac{\sqrt \ab}{\alpha^2}+\frac{1}{\alpha^2 \eps}}$    [Theorem~\ref{thm:close-main}]\\ \hline
      \end{tabular}
    \caption{\label{fig:badass:table} Summary of the sample complexity bounds for the problems of identity testing, and closeness testing of discrete distributions.}
\end{table}

%% file: related-work.tex
\subsection{Related Work}

A number of papers have recently studied hypothesis testing problems under differential privacy guarantees~\cite{WangLK15, GaboardiLRV16, RogersK17}. Some works analyze the distribution of the test statistic in the asymptotic regime. The work most closely related to ours is in~\cite{CaiDK17}, which studied identity testing in the finite sample regime. We mentioned their guarantees along with our results on identity testing in the previous section. 

There has been a line of research for statistical testing and estimation problems under the notion of \emph{local} differential privacy~\cite{wainwright2012privacy, duchi2013local,erlingsson2014rappor, pastore2016locally, kairouz2016discrete, WangHWNXYLQ16, YeB17}. These papers study some of the most basic statistical problems and also provide minimax lower bounds using Fano's inequality.~\cite{DiakonikolasHS15} study structured distribution estimation under differential privacy.

Information theoretic approaches to data privacy have been studied recently using quantities like mutual information, and guessing probability to quantify privacy~\cite{mir2012information, sankar2013utility, cuff2016differential, wang2016relation, IssaW17}.

In a contemporaneous and independent work,~\cite{aliakbarpour2017differentially}, the authors study the same problems that we consider, and obtain the same upper bounds for the sparse case, when $\ns\le\ab$. They also provide experimental results to show the performance of the privatized algorithms. However, their results are sub-optimal for $\ns=\Omega(\ab)$ for identity testing, and they do not provide any lower bounds for the problems. Both~\cite{CaiDK17}, and~\cite{diakonikolas2017sample} consider only pure-differential privacy, which are a special case of our results.

%% file: organization.tex
\subsection{Organization of the paper}

In Section~\ref{sec:preliminaries}, we discuss the definitions and notations. A general technique for proving lower bounds for differentially private algorithms is described in Section~\ref{sec:coupling}. In Section~\ref{sec:binary}, we study differentially private binary hypothesis testing as a warm-up. Section~\ref{sec:identity} gives upper and lower bounds for identity testing, and closeness testing is studied in Section~\ref{sec:close:}. Section~\ref{sec:upper-indentity} proves that the original uniformity tester of~\cite{Paninski08} is optimal in the sparse sample regime. 

%% file: preliminaries.tex
\section{Preliminaries} \label{sec:preliminaries}
We consider discrete distributions over a domain of size $\ab$, which we assume without loss of generality to be $[\ab]\ed\{1\upto\ab\}$. We denote length-$\ns$ samples $X_1\upto X_\ns$ by $\Xon$. For $\smb\in[\ab]$, let $\psmb$ be the probability of $\smb$ under $\p$. Let $\Mltsmb{\smb}{\Xon}$ be the number of times $\smb$ appears in $\Xon$. For $A\subseteq[\ab]$, let $\p(A) = \sum_{\smb\in A} \psmb$. Let $X\sim\p$ denote that the random variable $X$ has distribution $\p$. Let $\unifab$ be the uniform distribution over $[\ab]$, and $\Bern{b}$ be the Bernoulli distribution with bias $b$. 

\begin{definition}
The \emph{total variation} distance between distributions $\p$, and $\q$ over a discrete set $[\ab]$ is 
\[
\dtv{\p}{\q} \ed \sup_{A\subset[\ab]} \p(A) - \q(A) = \frac12\|\p-\q\|_1.
\]
\end{definition}

\begin{definition}
\label{def:coupling}
Let $\p$, and $\q$ be distributions over $\cX$, and $\cY$ respectively. A \emph{coupling} between $\p$ and $\q$ is a distribution over $\cX\times\cY$ whose marginals are $\p$ and $\q$ respectively.
\end{definition}

\begin{definition}
The \emph{Hamming distance} between two sequences $\Xon$ and $\Yon$ is 
$\ham{\Xon}{\Yon} \ed \sum_{i=1}^{\ns} \II\{{X_i\ne Y_i}\},$
 the number of positions where $\Xon$, and $\Yon$ differ. 
\end{definition}
\noindent We now define $(\eps,\delta)$-differential privacy. 
\begin{definition}
A randomized algorithm $\cA$ on a set $\cX^\ns\to S$ is said to be $(\eps,\delta)$-differentially private if for any $S\subset \text{range}(\cA)$, and all pairs of $\Xon$, and $\Yon$ with $\ham{\Xon}{\Yon}\le 1$
\[
\probof{\cA(\Xon)\in S}\le e^{\eps}\cdot\probof{\cA(\Yon)\in S}+\delta.
\]
\end{definition}

The case when $\delta=0$ is called \emph{pure differential privacy}.  For simplicity, we denote pure differential privacy as $\eps$-differential privacy ($\eps$-DP). The next lemma states a relationship between $(\eps, \delta)$- differential privacy and $\eps$-differential privacy. The result is implicitly present in~\cite{CaiDK17}, but we state here for completeness.

\begin{lemma} \label{lm:epsdelta}
	There is an $(\eps, \delta)$-DP algorithm for a testing problem if and only if there is an $(O(\eps+\delta),0)$-DP algorithm for the same testing problem. 
\end{lemma}

\begin{proof}
The proof has two parts. 
\begin{itemize}
\item
The first is to show that any $(\eps+\delta,0)$-DP algorithm is also $(\eps,\delta)$-DP. This is perhaps folklore, and is shown below. Suppose $\cA$ is a $(\eps+\delta)$-differentially private algorithm. Then for any $\Xon$ and $\Yon$ with $\ham{\Xon}{\Yon}\le 1$ and any $S\subset \text{range}(\cA)$, we have
	\begin{align*}
		\probof{\cA(\Xon)\in S} \le e^{\eps+\delta}\cdot\probof{\cA(\Yon)\in S} 
											 = e^{\eps}\cdot\probof{\cA(\Yon)\in S} + (e^{\delta} - 1)\cdot e^{\eps}\probof{\cA(\Yon)\in S}. 
	\end{align*}
	If $e^{\eps}\cdot\probof{\cA(\Yon)\in S} > 1-\delta$, then $\probof{\cA(\Xon)\in S}\le1< e^{\eps}\cdot\probof{\cA(\Yon)\in S}+\delta$.
	Otherwise, $e^{\eps}\cdot\probof{\cA(\Yon)\in S} \leq 1 - \delta$. To prove $(e^{\delta}-1)\cdot e^{\eps}\cdot\probof{\cA(\Yon)\in S}<\delta$, it suffices to show $(e^{\delta} - 1)(1 - \delta) \le \delta$, which is equivalent to $e^{-\delta} \ge 1 - \delta$, completing the proof.
\item
Consider an $(\eps, \delta)$-DP algorithm with error probability at most 0.05. Consider an algorithm that finally flips the answer with probability 0.05. This algorithm has error probability at most 0.1, and for any input, each outcome has probability at least 0.05.~\cite{CaiDK17} essentially showed that this new algorithm is $(\eps+10\delta, 0)$-DP. 
\end{itemize}
\end{proof}


A notion that is often useful in establishing bounds for differential privacy is sensitivity, defined below. 
\begin{definition}
\label{def:sensitivity}
The \emph{sensitivity} of a function $f:[\ab]^{\ns}\to\RR$ is 
$\Delta(f)\ed \max_{\ham{\Xon}{\Yon}\le1} \absv{f(\Xon)-f(\Yon)}.$
\end{definition}

\begin{definition}
The \emph{sigmoid-function} $\sigma:\RR\to(0,1)$ is
$\sigma(x) \ed \frac1{1+\exp(-x)} = \frac{\exp(x)}{1+\exp(x)}.$
\end{definition}
We need the following result for the sigmoid function. 
\begin{lemma}\label{lem:sig-cont}
\begin{enumerate}
	\item  For all $x, \gamma\in\RR$, 
$\exp(-\absv{\gamma})\le \frac{\sigma(x+\gamma)}{\sigma(x)} \le \exp(\absv{\gamma})$.
\item
Let $0<\eta<\frac12$. Suppose $x\ge \log\frac1\eta$. Then $\sigma(x)>1-\eta$.
\end{enumerate}
\end{lemma}
\begin{proof}
	Since $\sigma(x)$ is an increasing function, it suffice to assume that $\gamma>0$. In this case, 
$\frac{\sigma(x+\gamma)}{\sigma(x)} = \exp(\gamma) \cdot \frac{{1+\exp(x)}}{{1+\exp(x+\gamma)}}< \exp(\gamma)$.
For the second part,  
$\sigma(x) = 1-\frac{1}{1+e^x}\ge 1-\frac1{e^x} \ge 1-\eta$.
\end{proof}
\paragraph{Identity Testing ({\tt IT}).} Given description of a probability distribution $\q$ over $[\ab]$, parameters $\alpha$, and $\eps$, and $\ns$ independent samples $\Xon$ from an unknown distribution $\p$. An algorithm $\cA$ is an $(\ab,\alpha)$ - identity testing algorithm for $\q$, if
\begin{itemize}
\item
when $\p=\q$, $\cA$ outputs ``$\p=\q$'' with probability at least 0.9, and
\item
when $\dtv{\p}{\q}\ge\alpha$, $\cA$ outputs ``$\p\ne\q$'' with probability at least 0.9.
\end{itemize}
Furthermore, if $\cA$ is $(\eps,0)$-differentially private, we say $\cA$ is an $(\ab,\alpha,\eps)$-identity testing algorithm. 

\begin{definition}
The sample complexity of DP-identity testing problem, denoted $\Sitd$, is the smallest $\ns$ for which there exists an $(\ab,\dist,\eps)$-identity testing algorithm $\cA$ that uses $\ns$ samples. When privacy is not a concern, we denote the sample complexity as $\Sitnp$. When $\q=\unifab$, the problem reduces to uniformity testing, and the sample complexity is denoted $\Sut{\ab}{\dist}$.
\end{definition}

\paragraph{Closeness Testing ({\tt CT}).} Given parameters $\alpha$, and $\eps$, and $\ns$ independent samples $\Xon$, and $\Yon$ from unknown distributions $\p$, and $\q$. An algorithm $\cA$ is an $(\ab,\alpha)$-closeness testing algorithm if
\begin{itemize}
\item
If $\p=\q$, $\cA$ outputs $\p=\q$ with probability at least 0.9, and
\item
If $\dtv{\p}{\q}\ge\alpha$, $\cA$ outputs $\p\ne\q$ with probability at least 0.9.
\end{itemize}
Furthermore, if $\cA$ is $(\eps,0)$-differentially private, we say $\cA$ is an $(\ab,\alpha,\eps)$-closeness testing algorithm. 
\begin{definition}
The sample complexity of an $(\ab,\dist,\eps)$-closeness testing problem, denoted $\Sct$, is the least values of $\ns$ for which there exists an $(\ab,\dist,\eps)$-closeness testing algorithm $\cA$. When privacy is not a concern, we denote the sample complexity of closeness testing as $\Sctnp$. 
\end{definition}

We note again that by Lemma~\ref{lm:epsdelta}, we need to only consider pure differential privacy for both upper and lower bounds.

%% file: coupling.tex
\section{Privacy Bounds Via Coupling} 
\label{sec:coupling}

Recall that \emph{coupling} between distributions $\p$ and $\q$  over $\cX$, and $\cY$, is a distribution over $\cX\times\cY$ whose marginal distributions are $\p$ and $\q$ (Definition~\ref{def:coupling}). For simplicity, we treat coupling as a randomized function $f: \cX\to\cY$ such that if $X\sim\p$, then $Y=f(X)\sim\q$. Note that $X$, and $Y$ are not necessarily independent. 

\begin{example}
\label{exm:coin}
Let $\Bern{b_1}$, and $\Bern{b_2}$ be Bernoulli distributions with bias $b_1$, and $b_2$ such that $b_1<b_2$. Let $\p$, and $\q$ be distributions over $\{0,1\}^\ns$ obtained by $\ns$ \emph{i.i.d.} samples from $\Bern{b_1}$, and $\Bern{b_2}$ respectively. Let $\Xon$ be distributed according to $\p$. Generate a sequence $\Yon$ as follows: If $X_i=1$, then $Y_i=1$. If $X_i=0$, we flip another coin with bias $(b_2-b_1)/(1-b_1)$, and let $Y_i$ be the output of this coin. Repeat the process independently for each $i$, such that the $Y_i$'s are all independent of each other. Then $\probof{Y_i=1} = b_1 +(1-b_1) (b_2-b_1)/(1-b_1) = b_2$, and $\Yon$ is distributed according to $\q$.
\end{example}

We would like to use coupling to prove lower bounds on differentially private algorithms for testing problems. Let $\p$ and $\q$ be distributions over $\cX^\ns$. If there is a coupling between $\p$ and $\q$ with a small \emph{expected} Hamming distance, we might expect that the algorithm cannot have strong privacy guarantees. The following theorem formalizes this notion, and will be used to prove sample complexity bounds of differentially private algorithms. 
\begin{theorem}
\label{thm:coupling}
Suppose there is a coupling between distributions $\p$ and $\q$ over $\cX^\ns$, such that $\expectation{\ham{\Xon}{\Yon}} \le D$. Then, any $\eps$-differentially private algorithm that distinguishes between $\p$ and $\q$ with error probability at most $1/10$ must satisfy $D = \Omega\Paren{\frac1{\eps}}$.
\end{theorem}

\begin{proof}
Let $(\Xon, \Yon)$ be a coupling of $\p$, and $\q$ with $\expectation{\ham{\Xon}{\Yon}} \le D$. Then,  
$\probof{\alg(\Xon)=\p} \ge 0.9,\text{ and } \probof{\alg\Paren{\Yon}=\q} \ge 0.9$
implies that 
$\probof{\alg\Paren{\Xon}=\p \cap \alg\Paren{\Yon}=\q} \ge 0.9+0.9-1 = 0.8.$
By Markov's inequality, 
$\probof{\ham{\Xon}{\Yon}>10D}<\probof{\ham{\Xon}{\Yon}>10\cdot\expectation{\ham{\Xon}{\Yon}}}<0.1$.
Therefore, 
\begin{align}
\probof{\alg\Paren{\Xon}=\p \cap \alg\Paren{\Yon}=\q \cap \ham{\Xon}{\Yon}<10D} \ge 0.8+0.9-1 = 0.7.\label{eqn:dp-ham}
\end{align}

The condition of differential privacy states that for any $\Xon$, and $\Yon$, 
\[
e^{-\eps\cdot\ham{\Xon}{\Yon}}<\frac{\probof{\alg\Paren{\Xon}=\p}}{\probof{\alg\Paren{\Yon}=\p}}<e^{\eps\cdot\ham{\Xon}{\Yon}}.
\]
Consider one sequence pair $\Xon$, and $\Yon$ that satisfies~\eqref{eqn:dp-ham}. Then, we know that $\probof{\alg\Paren{\Xon}=\p} >0.7$, and $\probof{\alg\Paren{\Yon}=\q}>0.7$. By the condition of differential privacy,
\begin{align}
0.3 \ge {\probof{\alg\Paren{\Yon}=\p}}\ge {\probof{\alg\Paren{\Xon}=\p}}\cdot e^{-\eps\cdot\ham{\Xon}{\Yon}} = 0.7\cdot e^{-10\eps D}.\nonumber
\end{align}
Taking logarithm we obtain 
\[
D\ge \frac{\ln(7/3)}{10} \frac1{\eps} = \Omega\Paren{\frac1{\eps}},
\]
completing the proof.
\end{proof}

%% file: binary-upper.tex
\section{Binary Identity Testing}
\label{sec:binary}

We start with a simple testing problem. Given $\bias_0$, $\dist>0$, and $\eps>0$, and samples $\Xon\in\{0,1\}^\ns$ from $\Bern{\bias}$, distinguish between the cases $\bias=\bias_0$, and $\absv{\bias-\bias_0}\ge\dist$. We prove the following theorem (stated for $\eps$-DP without loss of generality). 

\begin{theorem}
	\label{thm:bina-main}
	Given $\bias_0\in[0,1]$, and $\eps>0, \delta \ge 0$, and $\alpha>0$. There is an $(\eps,\delta)$-DP algorithm that takes $O\Paren{\frac1{\alpha^2}+\frac1{\alpha\eps }}$ samples from a distribution $\Bern{\bias}$ and distinguishes between $\bias=\bias_0$, and $\absv{\bias-\bias_0}\ge\dist$ with probability at least $9/10$. Moreover, any algorithm for this task requires $\Omega\Paren{\frac1{\alpha^2}+\frac1{\alpha\eps}}$ samples.
\end{theorem}

Simple bias and variance arguments show that the sample complexity of this problem is $\Theta(1/\dist^2)$. In this section, we study the sample complexity with privacy constraints. 
We note that the upper bound can simply be achieved by the well known Laplace mechanism in differential privacy. We add a $Lap(1/\eps)$ random variable to the number of 1's in $\Xon$, and then threshold the output appropriately. The privacy is guaranteed by privacy guarantees of the Laplace mechanism. A small bias variance computation also gives the second term. For completeness, we provide a proof of the upper bound using our techniques in Section~\ref{sec:upper-bin}. The lower bound is proved using the coupling defined in Example~\ref{exm:coin} with Theorem~\ref{thm:coupling}. 


%% file: binary-lower-coupling.tex
\subsection{Binary Testing Lower Bound via Coupling}

Suppose $\bias_0=0.5$. Then least $\Omega\Paren{1/\alpha^2}$ samples are necessary to test whether $\bias = \bias_0$, or $\absv{\bias-\bias_0}>\dist$. We will prove the second term, namely a lower bound of $\Omega\Paren{\frac{1}{\dist\eps}}$ using a coupling. 

Consider the special case of Example~\ref{exm:coin} with $b_2=\half+\alpha$, and $b_1=\half$. Then, $D = (b_2-b_1)\ns = \alpha\ns$, and $\expectation{\ham{\Xon}{\Yon}} = \alpha \ns$. 
Applying Theorem~\ref{thm:coupling}, we know that any $\eps$-DP algorithm must satisfy
$\expectation{\ham{\Xon}{\Yon}}\ge\Omega\Paren{\frac1{\eps}}$, which implies that 
$\ns \ge \Omega\Paren{\frac1{\dist\eps}}.$

%% file: uniformity.tex
\section{Identity Testing}

 \label{sec:identity}

In this section, we prove the bounds for identity testing. Our main result is the following (stated for $\eps$-DP without loss of generality): 
\begin{theorem}
\label{thm:main-identity}
\[
    \Sitd =
    \begin{cases}
                                   \Theta\Paren{\frac{\sqrt k}{\alpha^2} + \frac{\ab^{1/2}}{\alpha\eps^{1/2}}}, & \text{when $m \le k$} \\
                                   \Theta\Paren{\frac{\sqrt k}{\alpha^2} +\frac{\ab^{1/3}}{\alpha^{4/3}\eps^{2/3}}}, & \text{when $k< m \le \frac{k}{\alpha^2}$} \\
                                   \Theta\Paren{\frac{\sqrt k}{\alpha^2} + \frac1{\alpha \eps}} & \text{when  $m \ge \frac{k}{\alpha^2}$.  } 
    \end{cases}
\]
We can combine the three bounds to obtain 
\[
\Sitd = \Theta\Paren{\frac{k^{1/2}}{\dist^2}+\max\left\{\frac{\ab^{1/2}}{\alpha\eps^{1/2}}, \frac{\ab^{1/3}}{\alpha^{4/3}\eps^{2/3}}, \frac1{\dist\eps}\right\}}.
\] 
\end{theorem}
Our bounds are tight up to constant factors in all parameters, including pure differential privacy when $\delta=0$. 

 For proving upper bounds, by Lemma~\ref{lm:epsdelta}, it suffices to prove them only in pure differential privacy case, which means $\Sit = O\Paren{\frac{\sqrt{\absz}}{\alpha^2}+\frac{\sqrt{k}}{\alpha \sqrt{\eps}} +\frac{ k^{ 1/3 } } { \alpha^{4/3}  \eps^{ 2/3 } } + \frac1{\alpha \eps}}$. In Theorem~\ref{thm:unif-identity} we will show a reduction from identity to uniformity testing under pure differential privacy. Using this, it will be enough to design algorithms for uniformity testing, which is done in Section~\ref{sec:upper-uniformity} where we will prove the upper bound. 
 
 Moreover since uniformity testing is a special case of identity testing, any lower bound for uniformity will port over to identity, and we give such bounds in Section~\ref{sec:lower-identity}.

%% file: uniformity-to-identity.tex
\subsection{Uniformity Testing implies Identity Testing}
\label{sec:iden-to-unif}

The sample complexity of testing identity of any distribution is $O(\frac{\sqrt{\ab}}{\alpha^2})$, a bound that is tight for the uniform distribution. Recently~\cite{goldreich2016uniform} proposed a scheme to reduce the problem of testing identity of distributions over $[\ab]$ for total variation distance $\alpha$ to the problem of testing uniformity over $[6k]$ with total variation parameter $\alpha/3$. In other words, they show that $\Sitnp\le \Sutnp{6\ab}{\dist/3}$. Our upper bounds use $(\eps+\delta, 0)$-DP, and therefore we only need to prove this result for pure differential privacy. Building up on the construction of~\cite{goldreich2016uniform}, we show that such a bound also holds for differentially private algorithms. 
\begin{theorem}
\label{thm:unif-identity}
\[
\Sit \le \Sut{6\ab}{\dist/3}.
\]
\end{theorem}
\noindent The theorem is proved in Section~\ref{app:iden-unif}. 

%% file: uniformity-upper_new2.tex
\subsection{Identity Testing -- Upper Bounds} \label{sec:upper-uniformity}

We had mentioned in the results that we can use the statistic of~\cite{Paninski08} to achieve the optimal sample complexity in the sparse case. This result is shown in Section~\ref{app:iden-lemma}. In this section, we will show that by privatizing the statistic proposed in~\cite{diakonikolas2017sample} we can achieve the sample complexity in Theorem~\ref{thm:main-identity} for all parameter ranges. The procedure is described in Algorithm~\ref{algorithm_uniformity}.



Recall that $\Mltsmb{\smb}{\Xon}$ is the number of appearances of $\smb$ in $\Xon$. Let
\begin{align}\label{eqn:dggp-stat}
S(\Xon) \ed \frac12 \cdot \sum_{x=1}^{n} \absv{ \frac {\Mltsmb{\smb}{\Xon} } {m} -\frac1{k} },
\end{align}
be the distance of the empirical distribution from the uniform distribution. Let $\mu(p) = \expectation{S(\Xon)}$ when the samples are drawn from distribution $p$. They show the following separation result on the expected value of $S(\Xon)$. 
\begin{lemma}[\cite{diakonikolas2017sample}]
Let $p$ be a distribution over $[\absz]$ and $\alpha = \dtv{\p}{\unifab}$, then there is a constant $c$ such that 
$$ \mu(\p) - \mu (\unifab) \ge  \left\{
\begin{array}{lll}

c{\alpha^2 \cdot  \frac{m^2}{k^2}},    &      & \text{when $m<k$}\\
c{\alpha^2 \cdot \sqrt { \frac{m}{k} }},  &      &  \text{when $k< m \le \frac{k}{\alpha^2}$} \\
c{\alpha},       &      & \text{when $~ m \ge \frac{k}{\alpha^2}$ }
\end{array} \right. $$\label{lem:diako}
\end{lemma}

\cite{diakonikolas2017sample} used this result to show that thresholding $S(\Xon)$ at 0 is an optimal algorithm for identity testing. Their result is stronger than what we require in our work, since we only care about making the error probability at most 0.1. We first normalize the statistic to simplify the presentation of our DP algorithm. Let
\begin{equation}\label{eqn:statistic}
    Z (\Xon) \ed
    \begin{cases}
                                   k \Paren {S(\Xon) - \mu (\unifab) - \frac1{2} c \alpha^2 \cdot  \frac{m^2}{k^2} ~}, & \text{when $m \le k$} \\
                                   m \Paren {S(\Xon) - \mu (\unifab) - \frac1{2} c \alpha^2 \cdot \sqrt { \frac{m}{k} } ~}, & \text{when $k< m \le \frac{k}{\alpha^2}$} \\
                                   m \Paren {S(\Xon) - \mu (\unifab)  - \frac1{2} c \alpha}, & \text{when  $m \ge \frac{k}{\alpha^2}$  } 
    \end{cases}
\end{equation}
where $c$ is the constant in Lemma~\ref{lem:diako}, and $\mu(\unifab)$ is the expected value of  ${S(\Xon)}$ when $\Xon$ are drawn from uniform distribution.

Therefore, for $\Xon$ drawn from $\unifab$, 
\begin{equation}\label{eqn:statistic-expectation-uniform}
    \expectation{Z (\Xon)} \le
    \begin{cases}
                                   - \frac1{2} c \alpha^2 \cdot  \frac{m^2}{k}, & \text{when $m \le k$} \\
                                    - \frac1{2} c \alpha^2 \cdot  { \frac{m^{3/2}}{k^{1/2}} }, & \text{when $k< m \le \frac{k}{\alpha^2}$} \\
                                  -\frac1{2} cm \alpha, & \text{when  $m \ge \frac{k}{\alpha^2}$ . } 
    \end{cases}
\end{equation}

For $\Xon$ drawn from $\p$ with $\dtv{\p}{\unifab}\ge\alpha$, 
\begin{equation}\label{eqn:statistic-expectation-far}
    \expectation{Z (\Xon)} \ge
    \begin{cases}
                                   \frac1{2} c \alpha^2 \cdot  \frac{m^2}{k}, & \text{when $m \le k$} \\
                                    \frac1{2} c \alpha^2 \cdot  { \frac{m^{3/2}}{k^{1/2}} }, & \text{when $k< m \le \frac{k}{\alpha^2}$} \\
                                  \frac1{2} c m\alpha, & \text{when  $m \ge \frac{k}{\alpha^2}$ . } 
    \end{cases}
\end{equation}

\begin{algorithm}
    \caption{Uniformity testing}
    \label{algorithm_uniformity}
    \hspace*{\algorithmicindent} \textbf{Input:}  $\eps$, $\dist$, $\delta$ be i.i.d. samples $\Xon$ from $\p$
    \begin{algorithmic}[1] 
    \State Let $Z(\Xon)$ be evaluated from~\eqref{eqn:dggp-stat}, and~\eqref{eqn:statistic}.
    ~~~~~~~~
    \State Generate $Y\sim\Bern{\sigma\Paren{\eps\cdot Z}}$, $\sigma$ is the sigmoid function.
    \State{ {\bf if} $Y=0$, {\bf return} $\p=\unifab$}
    \State{ {\bf else}, {\bf return} $\p=\unifab$}
\end{algorithmic}
\end{algorithm}

In order to prove the privacy bounds, we need the following (weak) version of the result of~\cite{diakonikolas2017sample}, which is sufficient to prove the sample complexity bound for constant error probability.
\begin{lemma}
\label{lem:identityerror}
There is a constant $C>0$, such that when $\ns>C\sqrt{\ab}/\dist^2$, then for $\Xon\sim \p$, where either $\p=\unifab$, or $\dtv{\p}{\unifab}\ge\alpha$, 
\[
\probof{\absv{{Z(\Xon)-\expectation{Z(\Xon)}}}> \frac{2\expectation{Z(\Xon)}}{3} }<0.01.
\]
\end{lemma}
\noindent The proof of this result is in Section~\ref{app:iden-lemma}.

We now prove that this algorithm is $\eps$-DP. We need the following sensitivity result. 
\begin{lemma}
$\Delta(Z)\le 1$ for all values of $m$, and $k$.  
\end{lemma}
\begin{proof}
Recall that $S(\Xon) \ed \frac12 \cdot \sum_{x=1}^{n} \absv{ \frac {\Mltsmb{\smb}{\Xon} } {m} -\frac1{k} }$. Changing any one symbol changes at most two of the $\Mltsmb{\smb}{\Xon}$'s. Therefore at most two of the terms change by at most $\frac1\ns$. Therefore, $\Delta(S(\Xon)\le\frac1\ns$, for any $\ns$. When $\ns\le \ab$, this can be strengthened with observation that $\Mltsmb{\smb}{\Xon}/\ns\ge \frac1\ab$, for all $\Mltsmb{\smb}{\Xon}\ge1$. Therefore, 
$S(\Xon) = \frac12\cdot \Paren{\sum_{\smb: \Mltsmb{\smb}{\Xon}\ge1}\Paren{\frac {\Mltsmb{\smb}{\Xon} } {m} -\frac1{k}}+ \sum_{\smb: \Mltsmb{\smb}{\Xon}=0} \frac1{k}} = \frac{\Phi_0(\Xon)}{\ab},$
where $\Phi_0(\Xon)$ is the number of symbols not appearing in $\Xon$. This changes by at most one when one symbol is changed, proving the result. 
\end{proof}

Using this lemma, $\eps\cdot Z(\Xon)$ changes by at most $\eps$ when $\Xon$ is changed at one location. Invoking Lemma~\ref{lem:sig-cont}, the probability of any output changes by a multiplicative $\exp(\eps)$, and the algorithm is $\eps$-differentially private.

We now proceed to prove the sample complexity bounds. Assume that $\ns>C\sqrt{\ab}/\alpha^2$, so Lemma~\ref{lem:identityerror} holds. Suppose $\eps$ be any real number such that $\eps|\expectation{Z(\Xon)}|>3\log {100}$. Let $\cA(\Xon)$ be the output of Algorithm~\ref{algorithm_uniformity}. Denote the output by 1 when $\cA(\Xon)$ is ``$\p\ne\unifab$'', and 0 otherwise. Consider the case when $\Xon\sim \p$, and $\dtv{\p}{\unifab}\ge\alpha$. Then, 
\begin{align*}
\probof{\cA(\Xon) = 1} \ge & \probof{\cA(\Xon) = 1 \text{ and } Z(\Xon)>\frac{\expectation{Z(\Xon)}}3}\\
& =\probof{Z(\Xon)>\frac{\expectation{Z(\Xon)}}3}\cdot \probof{\cA(\Xon) = 1\vert Z(\Xon)>\frac{\expectation{Z(\Xon)}}3}\\
& \ge 0.99\cdot \probof{B\Paren{\sigma\Paren{\eps\cdot \frac{\expectation{Z(\Xon)}}3}}=1}\\
&\ge 0.99 \cdot 0.99 \ge 0.9,
\end{align*}
where the last step uses that $\eps \expectation{Z(\Xon)}/3>\log {100}$, along with Lemma~\ref{lem:sig-cont}. The case of $\p=\unifab$ follows from the same argument. 

Therefore, the algorithm is correct with probability at least $0.9$, whenever, $\ns>C\sqrt{\ab}/\dist^2$, and $\eps|\expectation{Z(\Xon)}|>3\log {100}$. By ~\eqref{eqn:statistic-expectation-far}, note that $\eps|\expectation{Z(\Xon)}|>3\log {100}$ is satisfied when, 
\begin{align*}
c \alpha^2 \cdot  {m^2}/{k}\ge&(6\log {100})/\eps ,\ \text{ for $m \le k$}, &\\
 c \alpha^2 \cdot  { {m^{3/2}}/{k^{1/2}} }\ge&(6\log {100})/\eps,  \ \text{ for $k< m \le {k}/{\alpha^2}$}, &\\
                                 c \alpha\cdot \ns\ge& (6\log {100})/\eps,\ \text{ for  $m \ge {k}/{\alpha^2}$}.&
\end{align*}
This gives the upper bounds for all the three regimes of $\ns$.

%% file: uniformity-lower_new.tex
\subsection{Sample Complexity Lower bounds for Uniformity Testing}\label{sec:lower-identity}

In this section, we will show that for any value of $\ab,\dist,\eps$, 
\[
\Sitd= \Omega \Paren{\frac{k^{1/2}}{\dist^2}+\max\left\{\frac{\ab^{1/2}}{\alpha\eps^{1/2}}, \frac{\ab^{1/3}}{\alpha^{4/3}\eps^{2/3}}, \frac1{\dist\eps}\right\}},
\]
which can be rewritten as:
\[
\Sitd =
\begin{cases}
\Omega\Paren{\frac{\sqrt k}{\alpha^2} + \frac{\ab^{1/2}}{\alpha\eps^{1/2}}}, & \text{when $m \le k$} \\
\Omega\Paren{\frac{\sqrt k}{\alpha^2} +\frac{\ab^{1/3}}{\alpha^{4/3}\eps^{2/3}}}, & \text{when $k< m \le \frac{k}{\alpha^2}$} \\
\Omega\Paren{\frac{\sqrt k}{\alpha^2} + \frac1{\alpha \eps}} & \text{when  $m \ge \frac{k}{\alpha^2}$.  } 
\end{cases}
\]
The first term is the lower bound without privacy constraints, proved in~\cite{Paninski08}. In this section, we will prove the terms associated with privacy. 

The simplest argument is for $m \ge \frac{k}{\alpha^2}$. From Theorem~\ref{thm:bina-main}, $\frac1{\dist\eps}$ is a lower bound for binary identity testing, which is a special case of identity testing for distributions over $[k]$ (when $k-2$ symbols have probability zero). This proves the lower bound for this case. 

We now consider the cases $m \le k$ and $k< m \le \frac{k}{\alpha^2}$.


To this end, we invoke LeCam's two point theorem, and design a hypothesis testing problem that will imply a lower bound on uniformity testing. The testing problem will be to distinguish between the following two cases. 

{\bf Case 1:} We are given $\ns$ independent samples from the uniform distribution $\unifab$.

{\bf Case 2:} Generate a distribution $\p$ with $\dtv{\p}{\unifab}\ge\dist$ according to some prior over all such distributions. We are then given $\ns$ independent samples from this distribution $\p$. 

Le Cam's two point theorem~\cite{Yu97} states that any lower bound for distinguishing between these two cases is a lower bound on identity testing problem. 

We now describe the prior construction for {\bf Case 2}, which is the same as considered by~\cite{Paninski08} for lower bounds on identity testing without privacy considerations. For each $\textbf{z} \in\{\pm1\}^{\ab/2}$, define a distribution $\p_{\textbf{z}}$ over $[\ab]$ such that 
\begin{align}
\p_{\textbf{z}}(2i-1) = \frac{1+\textbf{z}_i\cdot 2\dist}{\ab}, \text{ and } \p_{\textbf{z}}(2i) = \frac{1-\textbf{z}_i\cdot 2\dist}{\ab}.\nonumber
\end{align}
Then for any $\textbf{z}$, $\dtv{P_{\textbf{z}}}{\unifab}= \alpha$. For {\bf Case 2}, choose  $\p$ uniformly from these $2^{\ab}/2$ distributions. Let $Q_2$ denote the distribution on $[\ab]^\ns$ by this process. In other words, $Q_2$ is a mixture of product distributions over $[\ab]$. 

In {\bf Case 1}, let $Q_1$ be the distribution of $\ns$ $i.i.d.$ samples from $\unifab$. 

To obtain a sample complexity lower bound for distinguishing the two cases, we will design a coupling between $Q_1$, and $Q_2$, and bound its expected Hamming distance. While it can be shown that the Hamming distance of the coupling between the uniform distribution with any \emph{one} of the $2^{\ab/2}$ distributions grows as $\alpha\ns$, it can be significantly smaller, when we consider the mixtures. In particular, the following lemma shows that there exist couplings with bounded Hamming distance. 

\begin{lemma}\label{lem:coupling-hamming}
	There is a coupling between $\Xon$ generated by $Q_1$, and $\Yon$ by $Q_2$ such that 
\begin{equation}\label{eqn:statistic-expectation-uniform}
   \expectation{\ham{\Xon}{\Yon}} \le
    \begin{cases}
                                   8\frac{\ns^2\dist^2}{\ab}, & \text{when $m \le k$} \\
                                    C\cdot \dist^2\frac{\ns^{3/2}}{{\ab}^{1/2}}, & \text{when $k< m \le \frac{k}{\alpha^2}$}.
    \end{cases}
\end{equation}
\end{lemma}
The lemma is proved in Appendix~\ref{app:lb-unif}. Now applying Theorem~\ref{thm:coupling},
\begin{enumerate}
	\item For $m \le k$, $8\frac{\ns^2\dist^2}{\ab} = \Omega\Paren{\frac1{\eps}}$, implying $m = \Omega \Paren{\frac{\ab^{1/2}}{\alpha\eps^{1/2}}}$.
	\item For $k< m \le \frac{k}{\alpha^2}$, $C\cdot \dist^2\frac{\ns^{3/2}}{{\ab}^{1/2}} = \Omega \Paren{\frac1{\eps}}$, implying $m = \Omega \Paren{\frac{\ab^{1/3}}{\alpha^{4/3}\eps^{2/3}}}$.
\end{enumerate}


%
%
%



%% file: closeness-testing.tex
\section{Closeness Testing} \label{sec:close:}

Recall the closeness testing problem from Section~\ref{sec:preliminaries}, and the tight non-private bounds from Table~\ref{fig:badass:table}. Our main result in this section is the following theorem characterizing the sample complexity of differentially private algorithms for closeness testing. 

\begin{theorem}
	\label{thm:close-main}
		If $\alpha>1/\ab^{1/4}$, and $\eps\dist^2>1/\ab$,
		\[
		\Sct= \Theta \Paren{ \frac{\ab^{2/3}} {\alpha^{4/3}} + \frac{\ab^{1/2}}{  \alpha \sqrt{\eps} }},
		\]
		otherwise, 
		\[
		\Sct= O \Paren{ \frac{\ab^{1/2}}{ \alpha^{2}} + \frac{1}{ \alpha^2 \eps}}.
		\]
		and 
		\[
		\Sct= \Omega \Paren{ \frac{\ab^{1/2}}{ \alpha^{2}} + \frac{\ab^{1/2} }{ \alpha\sqrt{\eps}}+ \frac{1}{ \alpha \eps}}.
		\]
\end{theorem}

This theorem shows that in the sparse regime, when $\ns=O(\ab)$, our bounds are tight up to constant factors in all parameters. 


\subsection{Closeness Testing -- Upper Bounds} \label{sec:close-upper}

In this section, we only consider the case when $\delta = 0$, which would suffice by lemma~\ref{lm:epsdelta}.

To prove the upper bounds, we privatize the closeness testing algorithm of~\cite{chan2014optimal}. To reduce the strain on the readers, we drop the sequence notations explicitly and let
\[
\mltx i \ed \Mltsmb{i}{\Xon}, \text{ and } \mlty i\ed \Mltsmb{i}{\Yon}. 
\]

Variants of the chi-squared test have been used to test closeness of distributions in the recent years~\cite{AcharyaDJOPS12, chan2014optimal}. In particular,  the statistic used by~\cite{chan2014optimal} is
\[
 Z(\Xon,\Yon) \ed \sum_{i\in[\ab]} \frac{(\mltx i-\mlty i)^2-\mltx i-\mlty i}{\mltx i+\mlty i},
\]
where we assume that $((\mltx i-\mlty i)^2-\mltx i-\mlty i)/({\mltx i+\mlty i})=0$, when $\mltx i +\mlty i = 0$.

The results in~\cite{chan2014optimal} were proved under Poisson sampling, and we also use Poisson sampling, with only a constant factor effect on the number of samples for the same error probability. They showed the following bounds:
\begin{align}
\expectation{ Z(\Xon,\Yon) } &= 0 \text{ when } \p=\q,\label{eqn:expectation-equal}\\
\Var\Paren{Z(\Xon,\Yon)} &\le 2\min\{\ab, \ns\} \text{ when } \p=\q,\label{eqn:variance-equal}\\
\expectation{ Z(\Xon,\Yon) } &\ge \frac{ \ns ^2\alpha^2}{4\ab+2\ns}  \text{ when } \dtv{\p}{\q}\ge\dist,\label{eqn:expectation-unequal}\\
\Var\Paren{Z(\Xon,\Yon)} &\le \frac1{1000}\expectation{Z(\Xon,\Yon)}^2 \text{ when } \p\ne\q, \text{ and } \ns=\Omega\Paren{\frac1{\alpha^2}}. \label{eqn:variance-equal}
\end{align}

 We use the same approach with the test statistic as with binary testing and uniformity testing to obtain a differentially private closeness testing method, described in Algorithm~\ref{algorithm_closeness}.
  \begin{algorithm}
  	\caption{}
  	\label{algorithm_closeness}
  	\hspace*{\algorithmicindent} \textbf{Input:}
  	$\eps$, $\alpha$, sample access to distribution $p$ and $q$
  	\begin{algorithmic}[1] 
  		\State $Z' \gets Z(\Xon,\Yon) - \frac12\frac{\ns^2 \alpha^2}{4\ab+2\ns}$
  		\State Generate $Y\sim\Bern{\sigma\Paren{\exp(\eps\cdot Z'}}$
    \State{ {\bf if} $Y=0$, {\bf return} $\p=\q$}
    \State{ {\bf else}, {\bf return} $\p\ne\q$}
  	\end{algorithmic}
  \end{algorithm}

We will show that Algorithm~\ref{algorithm_closeness} satisfies sample complexity upper bounds described in theorem~\ref{thm:close-main}.

\noindent We first bound the sensitivity (Definition~\ref{def:sensitivity}) of the test statistic to prove privacy bounds. 
\begin{lemma} \label{sens_close}
$\Delta(Z(\Xon, \Yon))\le14$. 
\end{lemma}
\begin{proof}
Since $Z(\Xon, \Yon)$ is symmetric, without loss of generality assume that one of the symbols is changed in $\Yon$. This would cause at most two of the $\mlty i$'s to change. Suppose $\mlty i\ge1$, and it changed to $\mlty i-1$. Suppose, $\mltx i+\mlty i > 1$, the absolute change in the $i$th term of the statistic is
\begin{align}
\absv{\frac{(\mltx i-\mlty i)^2}{\mltx i+\mlty i}  - \frac{(\mltx i-\mlty i+1)^2}{\mltx i+\mlty i-1} }
=&    \absv{\frac{(\mltx i+\mlty i) (2\mltx i-2\mlty i+1) +(\mltx i-\mlty i)^2} {(\mltx i+\mlty i) (\mltx i+\mlty i-1)  } } \nonumber \\
\leq&  \absv{\frac{2\mltx i-2\mlty i+1}{\mltx i+\mlty i-1} }+ \absv {\frac{\mltx i-\mlty i}{\mltx i+\mlty i-1} }\nonumber \\
\leq&   \frac{3\absv{\mltx i-\mlty i} +1}{\mltx i+\mlty i-1} \nonumber \\
\leq&    3+ \frac{4}{\mltx i+\mlty i-1}\leq  7.\nonumber
\end{align}
When $\mltx i+\mlty i = 1$, the change can again be bounded by 7.  Since at most two of the $\nu_i$'s change, we obtain the desired bound.
\end{proof}

\noindent Since the sensitivity of the statistic is at most 14, the input to the sigmoid changes by at most $14\eps$ when any input sample is changed.  Invoking Lemma~\ref{lem:sig-cont}, the probability of any output changes by a multiplicative $\exp(14\eps)$, and the algorithm is $14\eps$-differentially private. 

We now prove the correctness of the algorithm:

\textbf{Case 1: $\dist^2>\frac1{\sqrt \ab}$, and $\dist^2\eps>\frac1{\ab}$.} 
In this case, we will show that $\Sct  = O \Paren{ \frac{\ab^{2/3}}{\alpha^{4/3}} + \frac{\ab^{1/2}}{\alpha\sqrt\eps}}$. In this case, $\frac{\ab^{2/3}}{\alpha^{4/3}} + \frac{\ab^{1/2}}{\alpha\sqrt\eps}  \le 2\ab$.

 We consider the case when $\p = \q$, then $\Var\Paren{Z(\Xon,\Yon)} \le 2\min\{\ab, \ns\}$. Let $\Var\Paren{Z(\Xon,\Yon)} \le c \ns$ for some constant $c$. By the Chebychev's inequality,

\begin{align}
\probof{Z' > -\frac1{6} \cdot \frac{\ns^2 \alpha^2}{4\ab+2\ns}} \le& 
\probof{Z(\Xon, \Yon)  -\expectation{Z (\Xon, \Yon)} > \frac13 \cdot \frac{\ns^2 \alpha^2}{4\ab+2\ns}}  \nonumber\\
\le&\probof{Z(\Xon, \Yon)  -\expectation{Z (\Xon, \Yon)} > \frac13 \cdot \frac{\ns^2 \alpha^2}{8\ab}}  \nonumber\\
\le&\probof{Z(\Xon, \Yon)  -\expectation{Z (\Xon, \Yon)} >  \Paren{c\ns}^{1/2} \cdot \frac{\ns^{3/2} \alpha^2}{24c^{1/2} \ab}}  \nonumber\\
\le & 576c\cdot \frac{\ab^2}{\ns^3 \dist^4},\nonumber
\end{align}
where we used that $4k+2m\le 8\ab$. 

Therefore, there is a $C_1$ such that if $\ns \ge C_1 \ab^{2/3} / \alpha^{4/3}$, 
then under $\p=\q$, $\probof{Z' > -\frac1{6} \cdot \frac{\ns^2 \alpha^2}{4\ab+2\ns}}$ is at most 1/100. Now furthermore, if $\eps\cdot\ns^2\dist^2 / 48\ab>\log(20)$, then for all $Z' <  -\frac1{6} \cdot \frac{\ns^2 \alpha^2}{4\ab+2\ns}$, with probability at least 0.95, the algorithm outputs the $\p=\q$. Combining the conditions, we obtain that there is a constant $C_2$ such that for $\ns = C_2 \Paren{ \frac{\ab^{2/3}}{\alpha^{4/3}} + \frac{\ab^{1/2}}{\alpha\sqrt\eps}} $, with probability at least 0.9, the algorithm outputs the correct answer when the input distributions satisfy $\p=\q$. The case of  $\dtv{\p}{\q} > \dist$ distribution is similar and is omitted.

\textbf{Case 2: $\dist^2<\frac1{\sqrt \ab}$, or $\dist^2\eps<\frac1{\ab}$.} In this case, we will prove a bound of $O \Paren{ \frac{\sqrt\ab}{\alpha^{2}} + \frac{1}{\alpha^2\eps}}$ on the sample complexity. We still consider the case when $\p = \q$. We first note that when $\dist^2<\frac1{\sqrt \ab}$, or $\dist^2\eps<\frac1{\ab}$, then either $\frac{\sqrt\ab}{\alpha^{2}}+\frac{1}{\alpha^2\eps}>\ab$.  Hence we can assume that the sample complexity bound we aim for is at least $\Omega(\ab)$. So $\Var\Paren{Z(\Xon,\Yon)} \le c \ab$ for constant $c$. By the Chebychev's inequality,

\begin{align}
\probof{Z' > -\frac1{6} \cdot \frac{\ns^2 \alpha^2}{4\ab+2\ns}} \le& 
\probof{Z(\Xon, \Yon)  -\expectation{Z (\Xon, \Yon)} > \frac13 \cdot \frac{\ns^2 \alpha^2}{4\ab+2\ns}}  \nonumber\\
\le&\probof{Z(\Xon, \Yon)  -\expectation{Z (\Xon, \Yon)} > \frac13 \cdot \frac{\ns \alpha^2}{6}}  \nonumber\\
\le&\probof{Z(\Xon, \Yon)  -\expectation{Z (\Xon, \Yon)} >  \Paren{c \ab}^{1/2} \cdot \frac{\ns \alpha^2}{18c^{1/2} \ab^{1/2} }}  \nonumber\\
\le & 144\cdot c\cdot \frac{\ab}{\ns^2 \dist^4}.\nonumber
\end{align}

Therefore, there is a $C_1$ such that if $\ns \ge C_1 \ab^{1/2} / \alpha^{2}$, 
then under $\p=\q$, $\probof{Z' > -\frac1{6} \cdot \frac{\ns^2 \alpha^2}{4\ab+2\ns}}$ is at most 1/100. In this situation, if $\eps\cdot\ns \dist^2 / 36>\log(20)$, then for all $Z' <  -\frac1{6} \cdot \frac{\ns^2 \alpha^2}{4\ab+2\ns}$, with probability at least 0.95, the algorithm outputs the $\p=\q$. Combining with the previous conditions, we obtain that there also exists a constant $C_2$ such that for $\ns = C_2 \Paren{ \frac{\sqrt\ab}{\alpha^{2}} + \frac{1}{\alpha^2\eps}}$, with probability at least 0.9, the algorithm outputs the correct answer when the input distribution is $\p=\q$. The case of  $\dtv{\p}{\q} > \dist$ distribution is similar and is omitted.

\subsection{Closeness Testing -- Lower Bounds}  \label{sec:close:lower}

To show the lower bound part of Theorem~\ref{thm:close-main}, we need the following simple result.

\begin{lemma}  \label{lem:close-to-ident}
	 $\Sitd \leq \Sct$. 
\end{lemma}

\begin{proof}
Suppose we want to test identity with respect to $\q$. Given $\Xon$ from $\p$, generate $\Yon$ independent samples from $\q$. If $\p=\q$, then the two samples are generated by the same distribution, and otherwise they are generated by distributions that are at least $\eps$ far in total variation. Therefore, we can simply return the output of an $(\ab,\dist,\eps)$-closeness testing algorithm on $\Xon$, and $\Yon$. 
\end{proof}

By Lemma~\ref{lem:close-to-ident} we know that a lower bound for identity testing is also a lower bound on closeness testing.

We first consider the sparse case, when $\dist^2>\frac1{\sqrt \ab}$, and $\dist^2\eps>\frac1{\ab}$. In this case, we show that 
\[
\Sctd= \Omega\Paren{\frac{\ab^{2/3}}{\alpha^{4/3}}+\frac{\sqrt \ab}{\alpha\sqrt{\eps}}}.
\]
When $\alpha>\frac1{k^{1/4}}$, $\frac{\ab^{2/3}}{\alpha^{4/3}}$ is the dominating term in the sample complexity $\Sctnp = \Theta\Paren{\frac{\ab^{2/3}}{\alpha^{4/3}}+\frac{\sqrt \ab}{\alpha^{2}}}$, giving us the first term. By Lemma~\ref{lem:close-to-ident} we know that a lower bound for identity testing is also a lower bound on closeness testing giving the second term, and the lower bound of Theorem~\ref{thm:main-identity} contains the second term as a summand.

In the dense case, when $\dist^2<\frac1{\sqrt \ab}$, or $\dist^2\eps<\frac1{\ab}$, we show that 
\[
\Sctd = \Omega\Paren{\frac{\sqrt \ab}{\alpha^{2}}+\frac{\sqrt\ab}{\alpha\sqrt{\eps}}+\frac1{\dist\eps}}.
\]

In the dense case, using the non-private lower bounds of $\Omega\Paren{\frac{\ab^{2/3}}{\alpha^{4/3}}+\frac{\sqrt \ab}{\alpha^{2}}}$ along with the identity testing bound of sample complexity lower bounds of note that $\frac{\sqrt\ab}{\alpha\sqrt{\eps}}+\frac1{\dist\eps}$ gives a lower bound of $\Omega\Paren{\frac{\ab^{2/3}}{\alpha^{4/3}}+\frac{\sqrt \ab}{\alpha^{2}}+\frac{\sqrt\ab}{\alpha\sqrt{\eps}}+\frac1{\dist\eps} }$. However, in the dense case, it is easy to see that $\frac{\ab^{2/3}}{\alpha^{4/3}}=O\Paren{\frac{\sqrt \ab}{\alpha^{2}}+\frac{\sqrt\ab}{\alpha\sqrt{\eps}}}$ giving us the bound.

%% file: app-binary-upper.tex
\section{Upper Bound for Binary Testing}
\label{sec:upper-bin}

Our $\eps$-DP algorithm simply estimates the empirical bias, and decides if it is close to $\bias_0$. Let $\Mltsmb{1}{\Xon}$ be the number of one's in $\Xon$. Then, 
\begin{align}
\expectation{\Mltsmb{1}{\Xon}}=\ns\bias, \text{ and } \Var\Paren{\Mltsmb{1}{\Xon}}= \ns\bias(1-\bias).\label{eqn:bias-variance}
\end{align}
We compute the deviation of $\Mltsmb{1}{\Xon}$ from its expectation, and use it as our statistic:
$$ Z(\Xon) = \Mltsmb{1}{\Xon} - \ns\bias_0.$$

\begin{algorithm} \label{alg:bernoulli-identity1}
\textbf{Input:} $\eps$, $\dist$, $\bias_0$ i.i.d. samples $\Xon$ from $\Bern{\bias}$\\
\quad Generate $Y\sim \Bern{\sigma\Paren{\eps\cdot (\absv{Z(\Xon)} -\frac{\dist\ns}2)}}$  \\
  		\textbf{if} {$Y=0$}\\
  		 \textbf {return} $\p=\Bern{\bias_0}$\\
\textbf{else} \\
  		\textbf{return} ${\p}\ne\Bern{\bias_0}$
\caption{Binary Testing}
\end{algorithm}

\begin{lemma}
\label{thm:bin-upper}
Algorithm~\ref{alg:bernoulli-identity1} is an $\eps$-differentially private algorithm for testing if a binary distribution is $\Bern{\bias_0}$. Moreover, it has error probability at most 0.1, with sample complexity $O\Paren{\frac1{\alpha^2}+\frac1{\alpha\eps}}$.
\end{lemma}

\begin{proof}
We first prove the correctness. Consider the case when $\bias = \bias_0$. By Chebychev's inequality, 
\[
\probof{\absv{Z(\Xon)}\ge \beta\cdot\frac{\sqrt \ns}2}\le \probof{Z(\Xon)^2\ge \beta^2\ns\bias(1-\bias)}\le\frac1{\beta^2}.
\] 
For $\beta=10$, we have $\probof{\absv{Z(\Xon)}\ge \beta \cdot{\sqrt \ns}/2}<1/100$. 
Suppose $\ns$ satisfies $$\eps\frac{\dist \ns-\beta\sqrt \ns}2>\log \frac1{0.05},$$
then  with probability at least 99/100, $\eps\Paren{\frac{\dist\ns}2-\absv{Z(\Xon)}}>\log{20}$. Under this condition, the algorithm outputs $\p\ne \Bern{\bias_0}$ with probability at most 1/20. Therefore, the total probability of error is at most $.01+.05 < 0.1$. 
Note that there is a constant $C$, such that for  $\ns \ge C\Paren{\frac1{\alpha^2}+\frac1{\alpha\eps}}$, $\eps\Paren{\frac{\dist\ns}2-\absv{Z(\Xon)}} \ge \log(20)$. 
The case when $\absv{\bias-\bias_0}>\dist$  follows from similar arguments. 

We now prove the privacy guarantee. When one of the samples is changed, $Z(\Xon)$ changes by at most one, and $\eps\cdot (\absv{Z(\Xon)} -\frac{\dist\ns}2)$ changes by at most $\eps$. Invoking Lemma~\ref{lem:sig-cont}, the probability of any output changes by at most a multiplicative $\exp(\eps)$, and the algorithm is $\eps$-differentially private.
\end{proof}

%% file: app-uniformity-to-identity.tex
\section{Proof of Theorem~\ref{thm:unif-identity}}
\label{app:iden-unif}

\begin{proof} We first briefly describe the essential components of the construction of~\cite{goldreich2016uniform}. Given an explicit distribution $\q$ over $[\ab]$, there exists a randomized function $F_{\q}:[\ab]\to[6\ab]$  such that if $X\sim\q$, then $F_{\q}(X)\sim \unifabs{6\ab}$, and if $X\sim\p$ for a distribution with $\dtv{\p}{\q}\ge\dist$, then the distribution of $F_{\q}(X)$ has a total variation distance of at least $\dist/3$ from $\unifabs{6\ab}$. 
Given $s$ samples $X_1^s$ from a distribution $\p$ over $[\ab]$. Apply $F_\q$ independently to each of the $X_i$ to obtain a new sequence $Y_1^s = F_\q(X_1^s)\ed F_q(X_1)\ldots F_{\q}(X_s)$. Let $\cA$ be an algorithm that distinguishes $\unifabs{6\ab}$ from all distributions with total variation distance at least $\dist/3$ from it. Then consider the algorithm $\cA^{'}$ that outputs $\p=\q$ if $\cA$ outputs ``$\p = \unifabs{6k}$'', and outputs $\p\ne \q$ otherwise. This shows that without privacy constraints, $\Sitnp \le \Sutnp{6\ab}{\dist/3}$ (See~\cite{goldreich2016uniform} for details). 

We now prove that if further $\cA$ was an $\eps$-DP algorithm, then $\cA^{'}$ is also an $\eps$-DP algorithm. Suppose $X_1^s$, and $X_1^{'s}$ be two sequences in $[\ab]^s$ that could differ only on the last coordinate, namely $\Xos = \Xosmo X_s$, and $\Xosp = \Xosmo X_s^{'}$. 

Consider two sequences $\Yos= \Yosmo Y_s$, and $\Yosp= \Yosmo Y_s^{'}$ in $[6\ab]^s$ that could differ on only the last coordinate. Since $\cA$ is $\eps$-DP, 
\begin{align}
\cA(\Yos=\unifabs{6\ab}) \le \cA(\Yosp=\unifabs{6\ab})\cdot e^{\eps}.\label{eqn:dp-unif}
\end{align}
Moreover, since $F_{\q}$ is applied independently to each coordinate, 
\[
\probof{F_{\q}(\Xos) = \Yos} = \probof{F_{\q}(\Xosmo) = \Yosmo}\probof{F_{\q}(X_s) = Y_s}.
\]
Then, 
\begin{align}
&\ \ \ \ \probof{\cA^{'}(\Xos) = \q} \nonumber\\
& = \probof{\cA(F_{\q}(X_1^s)) = \unifabs{6\ab}}\nonumber\\
& = \sum_{\Yos}\ \probof{\mathcal{A}(Y^s_1) = \unifabs{6\ab}}\probof{F_{\q}(\Xos) = \Yos)}\nonumber\\
& = \sum_{\Yosmo} \sum_{Y_s\in[\ab]}\ \probof{\mathcal{A}(Y^s_1) = \unifabs{6\ab}} \probof{F_{\q}(\Xosmo) = \Yosmo}\probof{F_{\q}(X_s) = Y_s}\nonumber\\
& = \sum_{\Yosmo}\probof{F_{\q}(\Xosmo) = \Yosmo}\Brack{ \sum_{Y_s\in[\ab]}\ \probof{\mathcal{A}(Y^s_1) = \unifabs{6\ab}} \probof{F_{\q}(X_s) = Y_s}}.\label{eqn:first-term-prob}
\end{align}
Similarly, 
\begin{align}
\probof{\cA^{'}(\Xosp) = \q}
 \!\!=\!\! \sum_{\Yosmo}\probof{F_{\q}(\Xosmo)\!\! = \!\!\Yosmo}\Brack{ \sum_{Y_s^{'}\in[\ab]} \probof{\mathcal{A}(\Yosp) = \unifabs{6\ab}} \probof{F_{\q}(X_s^{'}) = Y_s^{'}}}.\label{eqn:second-term-prob}
\end{align}
For a fixed $\Yosmo$, the term within the bracket in~\eqref{eqn:first-term-prob}, and~\eqref{eqn:second-term-prob} are both expectations over the final coordinate. However, by~\eqref{eqn:dp-unif} these expectations differ at most by a multiplicative $e^{\eps}$ factor. This implies that
\begin{align}
	\probof{\cA'(\Xos) = \q}\le \probof{\cA'(\Xosp) = \q}e^{\eps}.\nonumber
\end{align}
The argument is similar for the case when the testing output is \textbf{not $\unifabs{6\ab}$}, and is omitted here. We only considered sequences that differ on the last coordinate, and the proof remains the same when any of the coordinates is changed. This proves the privacy guarantees of the algorithm.
\end{proof}

%% file: uniformity-upper.tex
\section{Identity Testing -- Upper Bounds}
\label{sec:upper-indentity}

In this section, we will show that we can also use the statistic of~\cite{Paninski08} to achieve the optimal sample complexity in the sparse case. 
By Theorem~\ref{thm:unif-identity}, any upper bound on uniformity testing is a bound on identity testing. To obtain differentially private algorithms, we need test statistic with small sensitivity. In the sparse regime, when $\ns=O(\ab)$,~\cite{Paninski08} gave such a statistic under uniformity testing.


With these arguments, we propose Algorithm~\ref{algorithm_uniformity_sparse} for testing uniformity and show it achieves the upper bound of Theorem~\ref{thm:main-identity} in sparse case.

Consider the sparse case when $\alpha>1/\ab^{1/4}$, and $\dist^2\eps>1/\ab$. We will prove an upper bound of $O\Paren{\frac{\sqrt{\absz}}{\alpha^2}+\frac{\sqrt{\absz}}{\alpha\sqrt{\eps}}}$. 

Recall that $\Mltsmb{\smb}{\Xon}$ is the number of appearances of $\smb$ in $\Xon$, and let
\[
\prfsmb{j}{\Xon} = \{\smb: \Mltsmb{\smb}{\Xon} = j\}
\] 
be the number of symbols appearing $j$ times in $\Xon$.~\cite{Paninski08} used $\prfsmb{1}{\Xon}$ as the statistic. For $\Xon\sim\unifab$,
\begin{align}
\expectation{\prfsmb{1}{\Xon}} = \ab\cdot {\ns\choose1} \frac1{\ab}\Paren{1-\frac1\ab}^{\ns-1} =  \ns\cdot\Paren{1-\frac1\ab}^{\ns-1}. \label{eqn:exp-prf-unif}
\end{align}
For $\Xon$ generated from a distribution $\p$ with $\dtv{\p}{\unif}\ge\alpha$,~\cite[Lemma 1]{Paninski08} showed that:
\begin{align}
\expectation{\prfsmb{1}{\Xon}} \le \ns\cdot\Paren{1-\frac1\ab}^{\ns-1} - \frac{\ns^2\alpha^2}{\ab}.\label{eqn:exp-prf-non-unif}
\end{align}
They also showed that 
\begin{align}
\Var\Paren{\prfsmb{1}{\Xon}}=O\Paren{\ns^2/\ab},
\label{eqn:variance-prf}
\end{align}
and used Chebychev's inequality to obtain the sample complexity upper bound of $O\Paren{\sqrt{\ab}/\eps^2}$ without privacy constraints.
We modify their algorithm slightly to obtain a differentially private algorithm. Let 
\begin{align}
Z(\Xon) =  \ns\cdot\Paren{1-\frac1\ab}^{\ns-1} - \prfsmb{1}{\Xon} - \frac{\ns^2\alpha^2}{2\ab}\label{eqn:stat-iden-sparse}
\end{align}
Suppose $\Xon\sim\p$, then
\[
\expectation{Z(\Xon)} = - \frac{\ns^2\alpha^2}{2\ab}, \text{ if $\p=\unifab$},
\]
and $\Xon\sim\p$, 
\[
\expectation{Z(\Xon)} \ge \frac{\ns^2\alpha^2}{2\ab}, \text{if $\dtv{\p}{\unifab}\ge\dist$}. 
\]

\begin{algorithm}
    \caption{Uniformity testing in the sparse sample regime}
    \label{algorithm_uniformity_sparse}
    \hspace*{\algorithmicindent} \textbf{Input:}  $\eps$, $\dist$, i.i.d. samples $\Xon$ from $\p$
    \begin{algorithmic}[1] 
    \State Let $Z$ be the value of the statistic in~\eqref{eqn:stat-iden-sparse}.
    ~~~~~~~~
    \State Generate $Y\sim\Bern{\sigma\Paren{\eps\cdot Z}}$, $\sigma$ is the sigmoid function.
    \If{$Y=0$}
    \State {\Return $\p=\unifab$}
    \Else 
    \State {\Return $\p\ne \unifab$}
    \EndIf
\end{algorithmic}
\end{algorithm}

We first prove the privacy bound. If we change  one symbol in $\Xon$, $\prfsmb{1}{\Xon}$ can change by at most 2, and therefore $\eps\cdot Z$ changes by at most $2\eps$.  Invoking Lemma~\ref{lem:sig-cont}, the probability of any output changes by a multiplicative $\exp(2\eps)$, and the algorithm is $2\eps$-differentially private.

The error probability proof is along the lines of binary testing. We first consider when $\p=\unifab$.  Using~\eqref{eqn:variance-prf} let $\Var\Paren{\prfsmb{1}{\Xon}}\le c\ns^2/\ab$ for a constant $c$. By the Chebychev's inequality
 \begin{align}
\probof{Z(\Xon)>-\frac{\ns^2\dist^2}{6\ab}} \le& 
\probof{\expectation{\prfsmb{1}{\Xon}}-\prfsmb{1}{\Xon}>\frac{\ns^2\dist^2}{3\ab}} \nonumber\\
\le& \probof{\expectation{\prfsmb{1}{\Xon}}-\prfsmb{1}{\Xon}>\frac{c\ns}{\sqrt \ab}\cdot\frac{\ns\dist^2}{3c\sqrt\ab}}\nonumber\\
\le & 9c^2\cdot\frac{\ab}{\ns^2\dist^4}.
\end{align}

Therefore, there is a $C_1$ such that if $\ns\ge C_1\sqrt{\ab}/\alpha^2$, then under the uniform distribution $\probof{Z(\Xon)>-\frac{\ns^2\dist^2}{6\ab}}$ is at most 1/100. Now furthermore, if $\eps\cdot\ns^2\dist^2/6k>\log(25)$, then for all $\Xon$ with $Z(\Xon)<-\frac{\ns^2\dist^2}{6\ab}$ with probability at least 0.95, the algorithm outputs the uniform distribution. Combining the conditions, we obtain that there is a constant $C_2$ such that for $\ns = C_2\Paren{\frac{\sqrt{\absz}}{\alpha^2}+\frac{\sqrt{\absz}}{\alpha\sqrt{\eps}}}$, with probability at least 0.9, the algorithm outputs uniform distribution when the input distribution is indeed uniform. The case of non-uniform distribution is similar since the variance bounds hold for both the cases, and is omitted.

%% file: uniformity-upper-lemma.tex
\section{Proof of  Lemma~\ref{lem:identityerror}}
\label{app:iden-lemma}

In order to prove the lemma, we need the following lemma, which is proved in \cite{diakonikolas2017sample}.

\begin{lemma}{(Bernstein version of McDiarmid's inequality)}
\label{lem:McDiarmid}
	Let $\Yon$ be independent random variables taking values in the set $\cY$. Let $f: \cY^m \rightarrow \RR$ be a function of $\Yon$ so that for every $j \in [m]$, and $y_1,...y_m,{y_j  ^{\prime}} \in \cY$, we have that:
$$\absv{f(y_1,...y_j,...y_m) - f(y_1,...,{y_j ^{\prime}},...y_m)} \le B, $$

Then we have

$$\probof{f-\expectation{f} \ge z} \le \exp\Paren{\frac{-2z^2}{mB^2}}.$$

In addition, if for each $j \in [m]$ and $y_1,...y_{j-1},y_{j+1},...y_m$ we have that 

$$ \mathrm{Var}_{Y_j} [f(y_1,...y_j,...y_m)] \le \sigma_j^2,$$

then we have 
$$\probof{f-\expectation{f} \ge z} \le \exp\Paren{\frac{-z^2}{\sum_{j=1}^{m} \sigma_j^2 +2Bz/3 }}.$$
\end{lemma}

The statistic we use $Z(X_1^m)$ has sensitivity at most 1, hence we can use $B =1$ in Lemma~\ref{lem:McDiarmid}. 

We first consider the case when $k< m \le \frac{k}{\alpha^2}$. When $p=\unifab$, we get $\expectation{Z(\Xon)} = - \frac1{2} c m \alpha^2 \cdot \sqrt { \frac{m}{k} }$, then by the first part of Lemma~\ref{lem:McDiarmid},
 \begin{align}
\probof{Z(\Xon)>\frac{\expectation{Z(\Xon)}}3} =& \probof{Z(\Xon)> -   \frac1{6} c m \dist^2 \cdot \sqrt{\frac{m}{k}}}  \nonumber\\
\le& \probof{Z(\Xon) - \expectation{Z(\Xon) } > \frac{2}{3}c m \dist^2 \cdot \sqrt{\frac{m}{k}}} \nonumber\\
\le & \exp\Paren{-\frac{8c^2m^2 \dist^4}{9k}}.
\end{align}

Therefore, there is a $C_1$ such that if $\ns\ge C_1\sqrt{\ab}/\alpha^2$, then under the uniform distribution $\probof{Z(\Xon)>\frac{\expectation{Z(\Xon)}}3}$  is at most 1/100. The non-uniform distribution part is similar and we omit the case. 

Then we consider the case when $\frac{k}{\alpha^2} < m$. When $p=\unifab$, we get $\expectation{Z(\Xon)} = - \frac1{2} c m \alpha$, then also by the first part of Lemma~\ref{lem:McDiarmid},
 \begin{align}
\probof{Z(\Xon)>\frac{\expectation{Z(\Xon)}}3} =& \probof{Z(\Xon)> -   \frac1{6} c m \dist} \nonumber\\   
\le& \probof{Z(\Xon) - \expectation{Z(\Xon) } > \frac{2}{3}c m \dist } \nonumber\\
\le & \exp\Paren{ -\frac{8c^2m \dist^2}{9} }.\nonumber
\end{align}

Using the same argument we can show that there is a constant $C_2$ such that for $\ns\ge C_2/\alpha^2$, then under the uniform distribution $\probof{Z(\Xon)>\frac{\expectation{Z(\Xon)}}3}$  is at most 1/100. The case of non-uniform distribution is omitted because of the same reason.

At last we consider the case when $m \le k$. In this case we need another result proved in~\cite{diakonikolas2017sample}:
\[
	\mathrm{Var}_{X_j} [Z(x_1, x_2, ...,X_j, .. ,x_m) ] \le \frac mk, \forall j, x_1,x_2,...,x_{j-1}, x_{j+1}, ... x_n
\]

 When $p=\unifab$, we get $\expectation{Z(\Xon)} = - \frac1{2} c k \alpha^2 \cdot \frac{m^2}{k^2}$, then by the second part of Lemma~\ref{lem:McDiarmid},
 \begin{align}
\probof{Z(\Xon)>\frac{\expectation{Z(\Xon)}}3} =& \probof{Z(\Xon)> -   \frac1{6} c \dist^2 \cdot \frac{m^2}{k}}   \nonumber\\
\le& \probof{Z(\Xon) - \expectation{Z(\Xon) } > \frac{2}{3}c  \dist^2 \cdot \frac{m^2}{k}} \nonumber\\
\le & \exp\Paren{ \frac{-\frac{4}{9}c^2 \dist^4 \frac{m^4}{k^2} }{ \frac{m^2}{k} + \frac{4}{9}c \dist^2 \frac{m^2}{k}} } \nonumber\\
\le & \exp\Paren{-\frac{2}{9}c \alpha^4 \frac{m^2}{k}}.\nonumber
\end{align}

Therefore, there is a $C_3$ such that if $\ns\ge C_3\sqrt{\ab}/\alpha^2$, then under the uniform distribution $\probof{Z(\Xon)>\frac{\expectation{Z(\Xon)}}3} $ is at most 1/100. The case of non-uniform distribution is similar and is omitted. 

Therefore, if we take $C = \max\{C_1, C_2, C_3\}$, we prove the result in the lemma.

%% file: lower-bound-coupling.tex
\section{Proof of Lemma~\ref{lem:coupling-hamming}}
\label{app:lb-unif}

\subsection{$\ns\le \ab$}
Before proving the lemma, we consider an example that will provide insights and tools to analyze the distributions $Q_1$, and $Q_2$.
Let $t\in\NN$. Let $P_2$ be the following distribution over $\{0,1\}^t$:
\begin{itemize}
\item
Select $b\in\{\half-\dist, \half+\dist\}$ with equal probability. 
\item
Output $t$ independent samples from $\Bern{b}$. 
\end{itemize}
Let $P_1$ be the distribution over $\{0,1\}^t$ that outputs $t$ independent samples from $\Bern{0.5}$.

When $t=1$, $P_1$ and $P_2$ both become $\Bern{0.5}$. For t=2, $P_1(00) = P_1(11) = \frac14+\dist^2$, and $P_1(10)= P_1(01)= \frac14-\dist^2$, and $\dtv{P_1}{P_2}$ is $2\dist^2$. A slightly general result is the following:
\begin{lemma}
\label{lem:p1-p2}
For $t=1$, $\dtv{P_1}{P_2}=0$ and for $t \geq 2$, $\dtv{P_1}{P_2} \le 2t\alpha^2$. 
\end{lemma}

\begin{proof}
Consider any sequence $X_1^{t}$ that has $t_0$ zeros, and $t_1= t-t_0$ ones. Then, 
\begin{align}
P_1(X_1^t) = {t\choose t_0} \frac1{2^t},\nonumber
\end{align}
and 
\begin{align}
P_2(X_1^t) = {t\choose t_0} \frac1{2^t}\Paren{\frac{(1-2\dist)^{t_0}(1+2\dist)^{t_1}+ (1+2\dist)^{t_0}(1-2\dist)^{t_1}}2}.\nonumber
\end{align}
The term in the parantheses above is minimized when $t_0 = t_1 = t/2$. In this case, 
\begin{align}
P_2(X_1^t) \ge & P_1(X_1^t)\cdot (1+2\dist)^{t/2} (1-2\dist)^{t/2}= P_1(X_1^t)\cdot (1-4\dist^2)^{t/2}. \nonumber
\end{align}
Therefore, 
\begin{align}
\dtv{P_1}{P_2} = \sum_{P_1>P_2} P_1(X_1^t)-P_2(X_1^t) \le \sum_{P_1>P_2} P_1(X_1^t)\Paren{1- (1-4\dist^2)^{t/2}} \le 2t\dist^2,\nonumber
\end{align}
where we used the Weierstrass Product Inequality, which states that $1-tx\le (1-x)^{t}$ proving the total variation distance bound.
\end{proof}

As a corollary this implies:
\begin{lemma}
\label{lem:coupling-bin}
There is a coupling between $\Xot$ generated from $P_1$ and $\Yot$ from $P_2$ such that $\expectation{\ham {\Xot}{\Yot}}\le t\cdot \dtv{P_1}{P_2} \le 4(t^2-t) \dist^2$.
\end{lemma}

\begin{proof}
Observe that $\sum_{X_1^t} \min\{P_1(X_1^t),P_2(X_1^t)\} = 1-\dtv{P_1}{P_2}$. Consider the following coupling between $P_1$, and $P_2$. Suppose $X_1^t$ is generated by $P_1$, and let $R$ be a $U[0,1]$ random variable. 
\begin{enumerate}
\item{ $R<1-\dtv{P_1}{P_2}$} Generate $X_1^t$ from the distribution that assigns probability  $\frac{\min\{P_1(X_1^t),P_2(X_1^t)\}}{1-\dtv{P_1}{P_2}}$ to $X_1^t$. Output $(X_1^t, X_1^t)$.
\item{$R\ge1-\dtv{P_1}{P_2}$}
Generate $X_1^t$ from the distribution that assigns probability ${\frac{P_1(X_1^t) - \min\{P_1(X_1^t),P_2(X_1^t)\}}{\dtv{P_1}{P_2}}}$ to $X_1^t$,  and $Y_1^t$ from the distribution that assigns  probability ${\frac{P_2(Y_1^t) - \min\{P_1(Y_1^t),P_2(Y_1^t)\}}{\dtv{P_1}{P_2}}}$  to $Y_1^t$ independently. Then output $(X_1^t, Y_1^t)$.
\end{enumerate}
To prove the coupling, note that the probability of observing $X_1^t$ is 
\[
\Paren{1-\dtv{P_1}{P_2}}\cdot \frac{\min\{P_1(X_1^t),P_2(X_1^t)\}}{1-\dtv{P_1}{P_2}}+\dtv{P_1}{P_2}\cdot {\frac{P_1(X_1^t) - \min\{P_1(X_1^t),P_2(X_1^t)\}}{\dtv{P_1}{P_2}}} = P_1(X_1^t).
\]
A similar argument gives the probability of $Y_1^t$ to be $P_2(Y_1^t)$.

\noindent Then  $\expectation{\ham {\Xot}{\Yot}}\le t \cdot \dtv{P_1}{P_2} = 2t^2 \dist^2\le 4(t^2-t)\alpha^2$ when $t\ge2$, and when $t=1$, the distributions are identical and the Hamming distance of the coupling is equal to zero.
\end{proof}

We now have the tools to prove Lemma~\ref{lem:coupling-hamming} for $\ns\le\ab$. 
\begin{proof}[Proof of Lemma~\ref{lem:coupling-hamming} for $\ns\le\ab$.]
The following is a coupling between $Q_1$ and $Q_2$:
\begin{enumerate}
\item 
Generate $\ns$ samples $Z_1^\ns$ from a uniform distribution over $[\ab/2]$. 
\item
For $j \in[\ab/2]$, let $T_j\subseteq[\ns]$ be the set of locations where $j$ appears. Note that $|T_j| = \Mltsmb{j}{Z_1^\ns}$. 
\item
To generate samples from $Q_1$:
\begin{itemize}
	\item Generate $|T_j|$ samples from a uniform distribution over $\{2j-1, 2j\}$, and replace the symbols in $T_j$ with these symbols. 
\end{itemize}
\item
To generate samples from $Q_2$:
\begin{itemize}
\item
Similar to the construction of $P_1$ earlier in this section, consider two distributions over $\{2j-1, 2j\}$ with bias $\half-\dist$, and $\half+\dist$. 
\item
Pick one of these distributions at random.
	\item Generate $|T_j|$ samples from it over $\{2j-1, 2j\}$, and replace the symbols in $T_j$ with these symbols. 
\end{itemize}
\end{enumerate}

From this process the coupling between $Q_1$, and $Q_2$ is also clear:
\begin{itemize}
\item 
Given $\Xon$ from $Q_2$, for each $j\in[\ab/2]$ find all locations $\ell$ such that $X_\ell= 2j-1$, or $X_{\ell}=2j$. Call this set $T_j$. 
\item
Perform the coupling between $P_2$ and $P_1$ from Lemma~\ref{lem:coupling-bin}, after replacing $\{0,1\}$ with $\{2j-1, 2j\}$.
\end{itemize}

	Using the coupling defined above, by the linearity of expectations, we get:
	\begin{align}
		\expectation{\ham{\Xon}{\Yon}}   =  \sum_{j=1}^{\ab/2} \expectation{\ham{X_1^{|T_j|}}{Y_1^{|T_j|}}} = \frac{\ab }{2} \expectation{\ham{X_1^{R}}{Y_1^{R}}} \le  \frac{\ab }{2}\cdot \expectation{4\dist^2(R^2-R)}, \nonumber
	\end{align}
	where $R$ is a binomial random variable with parameters $m$ and $2/\ab$. 
Now, a simple exercise computing Binomial moments shows that for $X\sim Bin(n, s)$, $\expectation{X^2-X} = s^2(n^2-n)\le n^2s^2.$ This implies that 
\[
\expectation{R^2-R} \le \frac{4\ns^2}{\ab^2}.
\]
Plugging this, we obtain
	\begin{align}
		\expectation{\ham{\Xon}{\Yon}} \le  \frac{\ab }{2}\cdot 	\frac{16\dist^2\ns^2}{\ab^2}  = \frac{8 m^2\dist^2}{\ab},\nonumber
	\end{align}
	proving the claim.
\end{proof}

%% file: lower-bound-coupling-medium.tex
\subsection{$\ab\le\ns\le \ab/\dist^2$} \label{sec:coup_unif_medium}

Lemma~\ref{lem:p1-p2} holds for all values of $t$, and $\alpha$. The lemma can be strengthened for cases where $\alpha$ is small. 

\begin{lemma}
	\label{lem:coup_p}
	Let $P_1$, and $P_2$ be the distributions over $\{0,1\}^t$ defined in the last section. 
	There is a coupling between $X_1^t$ generated by $P_1$, and $Y_1^t$ by $P_2$ such that 
	\[
	\expectation{\ham{X_1^t}{Y_1^t}} \le C\cdot (\alpha^2t^{3/2} +  \alpha^4 t^{5/2} + \alpha^5 t^3).
	\] 
\end{lemma}

\subsubsection{Proof of Lemma~\ref{lem:coupling-hamming} assuming Lemma~\ref{lem:coup_p}}

Given the coupling we defined in Section~\ref{sec:coup_p} for proving Lemma~\ref{lem:coup_p}, the coupling between $Q_1$, and $Q_2$ uses the same technique in the last section for $\ns\le \absz$. 
\begin{itemize}
	\item 
	Given $\Xon$ from $Q_2$, for each $j\in[\ab/2]$ find all locations $\ell$ such that $X_\ell= 2j-1$, or $X_{\ell}=2j$. Call this set $T_j$. 
	\item
	Perform the coupling in Section~\ref{sec:coup_p} between $P_2$ and $P_1$ on $T_j$, after replacing $\{0,1\}$ with $\{2j-1, 2j\}$.
\end{itemize}

Using the coupling defined above, by the linearity of expectations, we get:
\begin{align}
	\expectation{\ham{\Xon}{\Yon}}  & =  \sum_{j=1}^{\ab/2} \expectation{\ham{X_1^{|T_j|}}{Y_1^{|T_j|}}} \nonumber\\
	& = \frac{\ab }{2} \expectation{\ham{X_1^{R}}{Y_1^{R}}} \nonumber \\
	& \le  \frac{\ab }{2}\cdot \expectation{64\cdot\Paren{\alpha^4R^{5/2} + {\dist^2R^{3/2} + \alpha^5 R^3}} }, \nonumber
\end{align}
where $R \sim \bino{m}{2/\ab}$.

We now bound the moments of Binomial random variables. The bound is similar in flavor to~\cite[Lemma 3]{AcharyaOST17} for Poisson random variables.
\begin{lemma}
Suppose $\frac{m}{k}>1$, and $Y\sim \bino{m}{\frac1k}$, then for $\gamma \ge 1$, there is a constant $C_{\gamma}$ such that 
\[
\expectation{Y^{\gamma}}\le C_\gamma {\Paren{\frac{m}{k}}}^\gamma.
\]
\end{lemma}

\begin{proof}
	For integer values of $\gamma$, this directly follows from the moment fomula for Binomial distribution~\cite{Knob08}, and for other $\gamma \ge 1$, by Jensen's Inequality
	\[
		\expectation{Y^{\gamma}} \le \expectation{\Paren{Y^{\ceil{\gamma} }}^{\frac{\gamma}{\ceil{\gamma}}}}\le  \expectation{\Paren{Y^{\ceil{\gamma} }}} ^{\frac{\gamma}{\ceil{\gamma}}} \le \Paren{ C_{\ceil{\gamma}} \expectation{Y}^{\ceil{\gamma} }}^{\frac{\gamma}{\ceil{\gamma}}} = C' (\expectation{Y})^{\gamma} , 
	\]
	proving the lemma. 
	\end{proof}
Therefore, letting $C = \max \{C_{5/2}, C_{3}, C_{3/2}\}$, we obtain
\begin{align*}
\expectation{\ham{\Xon}{\Yon}}  \le 32k C\cdot\Paren{\alpha^4\Paren{\frac{m}{k}}^{5/2} + {\dist^2 \Paren{\frac{m}{k}}^{3/2} + {{\alpha^5 \Paren{\frac{m}{k}}}}^3}}
\end{align*}
Now, notice $\alpha\sqrt\frac{m}{k}<1$. Plugging this, 
\begin{align*}
\expectation{\ham{\Xon}{\Yon}}  \le\ & 32 C\cdot k\cdot\Paren{\alpha^4\Paren{\frac{m}{k}}^{5/2} + {\dist^2 \Paren{\frac{m}{k}}^{3/2} + {{\alpha^5 \Paren{\frac{m}{k}}}}^3}} \\
=\ & 32 C\cdot k \alpha^2\cdot\Paren{\alpha^2\frac{m}{k}\cdot\Paren{\frac{m}{k}}^{3/2} + \Paren{\frac{m}{k}}^{3/2} + {{\alpha^3 \Paren{\frac{m}{k}}^{3/2}\Paren{\frac{m}{k}}}}^{3/2}} \\
\le\ & 96 C\cdot k \Paren{\frac{m}{k}}^{3/2},
\end{align*}
completing the argument. 

\subsubsection{Proof of Lemma~\ref{lem:coup_p}} \label{sec:coup_p}

To prove Lemma~\ref{lem:coup_p}, we need a few lemmas first:

\begin{definition}
	A random variable $Y_1$ is said to stochastically dominate $Y_2$ if for all $t$,  $\probof{Y_1\ge t}\ge \probof{Y_2\ge t}$. 
\end{definition}

\begin{lemma}\label{stod}
	Suppose $N_1 \sim \bino{t}{\frac{1}{2}}, N_2 \sim \frac{1}{2}\bino{t}{\frac{1+\alpha}{2}} + \frac{1}{2}\bino{t}{\frac{1-\alpha}{2}} $. Then $Z_2 = \max \{N_2, t-N_2\}$ stochastically dominates $Z_1 = \max \{N_1, t-N_1\}$.
\end{lemma}

\begin{proof}
	\begin{equation}
		\probof{Z_2 \ge l} = \sum_{i = 0}^{t - l} {t \choose i} \left[\Paren{\frac{1+\dist}{2}}^i \Paren{\frac{1-\dist}{2}}^{t-i} + \Paren{\frac{1-\dist}{2}}^i \Paren{\frac{1+\dist}{2}}^{t-i} \right]\nonumber
	\end{equation}
	\begin{equation}
		\probof{Z_1 \ge l} = 2\cdot \sum_{i = 0}^{t - l} {t \choose i} \Paren{\frac12}^t\nonumber
	\end{equation}
	Define $F(l) =\probof{Z_2 \ge l} - \probof{Z_1 \ge l} $. What we need to show is $F(l) \ge 0, \forall l \ge \frac{t}{2}$. First we observe that $\probof{Z_2 \ge \frac{t}{2}} =  \probof{Z_1 \ge \frac{t}{2}} = 1$ and  $\probof{Z_2 \ge t} = (\frac{1 +\dist}{2})^t + (\frac{1 - \dist}{2})^t  \ge 2 (\frac12)^t  = \probof{Z_1 \ge t}$. Hence $F(\frac{t}{2}) = 0, F(t) > 0$.
	Let
	\[
		f(l) = F(l+1) - F(l)= - {{t \choose l} } \left[\Paren{\frac{1+\dist}{2}}^l \Paren{\frac{1-\dist}{2}}^{t-l} + \Paren{\frac{1-\dist}{2}}^l \Paren{\frac{1+\dist}{2}}^{t-l} - 2 \Paren{\frac12}^t\right].
	\]
	Let $g(x) = \Paren{\frac{1+\dist}{2}}^x \Paren{\frac{1-\dist}{2}}^{t-x} + \Paren{\frac{1-\dist}{2}}^x \Paren{\frac{1+\dist}{2}}^{t-x} - 2 \Paren{\frac12}^t, x \in [t/2, t]$, then
	\[
		\frac{dg(x)}{dx} =  \ln\Paren{\frac{1+\dist}{1-\dist}} \cdot\left[\Paren{\frac{1+\dist}{2}}^x \Paren{\frac{1-\dist}{2}}^{t-x} - \Paren{\frac{1-\dist}{2}}^x \Paren{\frac{1+\dist}{2}}^{t-x}\right] \ge 0
	\]
	
	We know $g(t/2) < 0, g(t) > 0$, hence $\exists x^*, s.t. g(x) \le 0, \forall x< x^*$ and $g(x) \ge 0, \forall x > x^*$. Because $f(l) =  - {t \choose l}  g(l)$, hence $\exists l^*, s.t. f(l) \le 0, \forall l \ge l^*$ and $f(l) \ge 0, \forall l < l^*.$
	Therefore, $F(l)$ first increases and then decreases, which means $F(l)$ achieves its minimum at $\frac{t}{2}$ or $t$. Hence $F(l) \ge 0$, completing the proof.
\end{proof}

For stochastic dominance, the following definition~\cite{den2012probability} will be useful.  
\begin{definition}
	A coupling $(X',Y')$ is a monotone coupling if $\probof{X' \ge Y'}=1$. 
\end{definition}

The following lemma states a nice relationship between stochastic dominance and monotone coupling, which is provided as Theorem~7.9 in \cite{den2012probability}

\begin{lemma} \label{stod:coup}
	Random variable $X$ stochastically dominates $Y$ if and only if there is a monotone coupling between $(X',Y')$ with $\probof{X' \geq Y'} = 1$. 
\end{lemma}

By Lemma~\ref{stod:coup}, there is a monotone coupling between $Z_1 = \max \{N_1, t-N_1\}$ and $Z_2 = \max \{N_2, t-N_2\}$. Suppose the coupling is $P^c_{Z_1,Z_2}$, we define the coupling between $X_1^t$ and $Y_1^t$ as following:

\begin{enumerate}
	\item Generate $X_1^t$ according to $P_2$ and count the number of one's in $X_1^t$ as $n_1$.
	\item Generate $n_2$ according to $P^c[Z_2| Z_1 = \max\{n_1, t- n_1\}]$.
	\item If $n_1 > t - n_1$, choose $n_2 - n_1$ of the zero's in $X_1^t$ uniformly at random and change them to one's to get $Y_1^t$ 
	\item If $n_1 < t - n_1$, choose $n_2 - (t - n_1)$ of the one's in $X_1^t$ uniformly at random and change them to zero's to get $Y_1^t$
	\item If $n_1 = t - n_1$, break ties uniformly at random and do the corresponding action.
	\item Output $(X_1^t,Y_1^t)$
\end{enumerate}

Since the coupling is monotone, and $\dtv{X_1^t}{Y_1^t} = Z_2 - Z_1$ for every pair of $(X_1^t,Y_1^t)$, we get:
\[
	\expectation{\dtv{X_1^t}{Y_1^t}} = \expectation{\max \{N_2, t-N_2\}} - \expectation{\max \{N_1, t-N_1\}}.
\]

Hence, to show lemma~\ref{lem:coup_p}, it suffices to show the following lemma:

\begin{lemma}
Suppose $N_1 \sim \bino{t}{\frac{1}{2}}, N_2 \sim \frac{1}{2}\bino{t}{\frac{1+\alpha}{2}} + \frac{1}{2}\bino{t}{\frac{1-\alpha}{2}} $.
\[
	\expectation{\max \{N_2, t-N_2\}} - \expectation{\max \{N_1, t-N_1\}} < C\cdot (\alpha^2t^{3/2} +  \alpha^4 t^{5/2} + \alpha^5 t^3).
\]
\end{lemma}
\begin{proof}
\begin{align}
	\ &\expectation{\max \{N_2, t-N_2\}} \nonumber\\
	 =\ & \sum_{0\le\ell\le t/2}(t/2 + \ell){t\choose \frac t2-\ell}\Paren{\Paren{\frac{1-\alpha}{2}}^{\frac t2-\ell}\Paren{\frac{1+\alpha}{2}}^{\frac t2+\ell}+\Paren{\frac{1+\alpha}{2}}^{\frac t2-\ell}\Paren{\frac{1-\alpha}{2}}^{\frac t2+\ell}} \nonumber\\
	=\ & \frac t2 +  \sum_{0\le\ell\le t/2} \ell{t\choose \frac t2-\ell}\Paren{\Paren{\frac{1-\alpha}{2}}^{\frac t2-\ell}\Paren{\frac{1+\alpha}{2}}^{\frac t2+\ell}+\Paren{\frac{1+\alpha}{2}}^{\frac t2-\ell}\Paren{\frac{1-\alpha}{2}}^{\frac t2+\ell}}. \nonumber
\end{align}

Consider a fixed value of $t$. Let 
\[
f(\alpha) = \sum_{0\le\ell\le t/2}\ell{t\choose \frac t2-\ell}\Paren{\Paren{\frac{1-\alpha}{2}}^{\frac t2-\ell}\Paren{\frac{1+\alpha}{2}}^{\frac t2+\ell}+\Paren{\frac{1+\alpha}{2}}^{\frac t2-\ell}\Paren{\frac{1-\alpha}{2}}^{\frac t2+\ell}}.
\]

The first claim is that this expression is minimized at $\alpha=0$. This is because of the monotone coupling between $Z_1$ and $Z_2$, which makes $\expectation{Z_2} \ge \expectation{Z_1}$. This implies that $f'(0)=0$, and by intermediate value theorem, there is $\beta\in[0,\alpha]$, such that 
\begin{align}
f(\alpha) = f(0)+\frac12\alpha^2\cdot f''(\beta).\label{eq:double-derivative}
\end{align}
We will now bound this second derivative. 
To further simplify, let
\[
g(\alpha) = \Paren{\frac{1-\alpha}{2}}^{\frac t2-\ell}\Paren{\frac{1+\alpha}{2}}^{\frac t2+\ell}+\Paren{\frac{1+\alpha}{2}}^{\frac t2-\ell}\Paren{\frac{1-\alpha}{2}}^{\frac t2+\ell}.
\]
Differentiating $g(\alpha)$, twice with respect to $\alpha$, we obtain, 
\begin{align}
g''(\alpha) = & ~\frac1{16} \cdot \Paren{\alpha^2(t^2-t)-4\alpha\ell(t-1)+4\ell^2-t}\Paren{\frac{1-\alpha}{2}}^{\frac t2-\ell-2}\Paren{\frac{1+\alpha}{2}}^{\frac t2+\ell-2}\nonumber\\
& + \frac1{16} \cdot \Paren{\alpha^2(t^2-t)+4\alpha\ell(t-1)+4\ell^2-t}\Paren{\frac{1+\alpha}{2}}^{\frac t2-\ell-2}\Paren{\frac{1-\alpha}{2}}^{\frac t2+\ell-2}.\nonumber
\end{align}

For any $\ell\ge0$,
\[
\Paren{\frac{1+\alpha}{2}}^{\frac t2-\ell-2}\Paren{\frac{1-\alpha}{2}}^{\frac t2+\ell-2}\le 
\Paren{\frac{1-\alpha}{2}}^{\frac t2-\ell-2}\Paren{\frac{1+\alpha}{2}}^{\frac t2+\ell-2}.
\]

Therefore $g''(\alpha)$ can be bound by,
\begin{align}
g''(\alpha) \le \frac1{16}\cdot\Paren{\alpha^2t^2+4\ell^2}\Paren{\Paren{\frac{1-\alpha}{2}}^{\frac t2-\ell-2}\Paren{\frac{1+\alpha}{2}}^{\frac t2+\ell-2}+\Paren{\frac{1-\alpha}{2}}^{\frac t2-\ell-2}\Paren{\frac{1+\alpha}{2}}^{\frac t2+\ell-2}}.\nonumber
\end{align}

When $\alpha<\frac14$, $(1-\alpha^2)^2>\frac12$, and we can further bound the above expression by
\[
g''(\alpha) 
\le 2\cdot\Paren{\alpha^2t^2+4\ell^2}\Paren{\Paren{\frac{1-\alpha}{2}}^{\frac t2-\ell}\Paren{\frac{1+\alpha}{2}}^{\frac t2+\ell}+\Paren{\frac{1-\alpha}{2}}^{\frac t2-\ell}\Paren{\frac{1+\alpha}{2}}^{\frac t2+\ell}}.
\]

Suppose $X$ is a $\bino{t}{ \frac{1+\beta}2}$ distribution. Then, for any $\ell>0$,
\[
\probof{\absv{X-\frac t2}=\ell} = {t\choose \frac t2-\ell}\Paren{\Paren{\frac{1-\beta}{2}}^{\frac t2-\ell}\Paren{\frac{1+\beta}{2}}^{\frac t2+\ell}+\Paren{\frac{1+\beta}{2}}^{\frac t2-\ell}\Paren{\frac{1-\beta}{2}}^{\frac t2+\ell}}.
\]
Therefore, we can bound~\eqref{eq:double-derivative}, by
\[
f''(\beta) \le 2\cdot\Paren{\beta^2t^2 \expectation{\absv{X-\frac t2}}+4\expectation{\absv{X-\frac t2}^3}}.
\] 
For $X\sim \bino{\ns}{r}$, 
\begin{align*}
\expectation{\Paren{X-\ns r}^2}=& \ns r(1-r)\le \frac{\ns}4, \text{ and }\\
\expectation{\Paren{X-\ns r}^4}=& \ns r(1-r)\Paren{3r(1-r) (\ns-2)+1}\le 3\frac{\ns^2}{4}.
\end{align*}
We bound each term using these moments, 
\begin{align}
\expectation{\absv{X-\frac t2}} \le \expectation{\Paren{X-\frac t2}^2}^{1/2}
=& \Paren{t\frac{(1-\beta^2)}4+ \Paren{\frac{t\beta}{2}}^2}^{1/2}\nonumber \le   \sqrt t + {t\beta}.
\end{align}
We similarly bound the next term, 
\begin{align}
\expectation{\absv{X-\frac t2}^3} \le  \expectation{\Paren{X-\frac t2}^4}^{3/4}\le & \expectation{\Paren{X- \frac{t(1+\beta)}{2}+\frac{t\beta}{2}}^4}^{3/4}\nonumber\\
\le &  8\Paren{\expectation{\Paren{X- \frac{t(1+\beta)}2}^4}^{3/4}+\Paren{\frac{t\beta}2}^3}\nonumber\\
\le & 8\Paren{t^{3/2} + \Paren{\frac{t\beta}2}^3},\nonumber
\end{align}
where we use $(a+b)^4\le 8(a^4+b^4)$. 
Therefore, 
\[
f''(\beta) \le 64\cdot\Paren{\beta^2t^{5/2} + {t^{3/2} + {(t\beta)}^3}} \le 64\cdot\Paren{\alpha^2t^{5/2} + {t^{3/2} + {(t\alpha)}^3}} 
\]
As a consequence,
\[
	\expectation{\max \{N_2, t-N_2\}} - \expectation{\max \{N_1, t-N_1\}} = \dist^2 f''(\beta) \le 64\cdot (\alpha^2t^{3/2} +  \alpha^4 t^{5/2} + \alpha^5 t^3).
\]
completing the proof.
\end{proof}